\documentclass[journal]{IEEEtran}
\usepackage{amsmath, amsfonts, amsthm}
\usepackage[linesnumbered,ruled,vlined]{algorithm2e}
\usepackage[caption=false,font=footnotesize,labelfont=sf,textfont=sf]{subfig}
\usepackage{booktabs}
\usepackage{verbatim}
\usepackage{graphicx}
\usepackage{threeparttable}
\usepackage[colorlinks, linkcolor=blue, citecolor=blue]{hyperref}

\newtheorem{theorem}{Theorem}
\newtheorem{proposition}{Proposition}
\newtheorem{lemma}{Lemma}
\newtheorem{remark}{Remark}
\newenvironment{psmallmatrix}{\left(\begin{smallmatrix}}{\end{smallmatrix}\right)}

\newcommand{\R}{\mathbb{R}}
\newlength{\oldtextfloatsep}
\newlength\mylen

\begin{document}

\title{Learnable Similarity and Dissimilarity Guided Symmetric Non-Negative Matrix Factorization}
\author{Wenlong Lyu, Yuheng Jia %
\thanks{Wenlong Lyu and Yuheng Jia are with the School of Computer Science and Engineering, Southeast University, Nanjing 211189, Jiangsu (email: l\_w\_l@seu.edu.cn; yhjia@seu.edu.cn), (\emph{Corresponding Author: Yuheng Jia})} %
}

\maketitle

\begin{abstract}
    Symmetric nonnegative matrix factorization (SymNMF) is a powerful tool for clustering, which typically uses the $k$-nearest neighbor ($k$-NN) method to construct similarity matrix.
    However, $k$-NN may mislead clustering since the neighbors may belong to different clusters, and its reliability generally decreases as $k$ grows. 
    In this paper, we construct the similarity matrix as a weighted $k$-NN graph with learnable weight that reflects the reliability of each $k$-th NN.
    This approach reduces the search space of the similarity matrix learning to $n - 1$ dimension, as opposed to the $\mathcal{O}(n^2)$ dimension of existing methods, where $n$ represents the number of samples.
    Moreover, to obtain a discriminative similarity matrix, we introduce a dissimilarity matrix with a dual structure of the similarity matrix, and propose a new form of orthogonality regularization with discussions on its geometric interpretation and numerical stability.
    An efficient alternative optimization algorithm is designed to solve the proposed model, with theoretically guarantee that the variables converge to a stationary point that satisfies the KKT conditions.
    The advantage of the proposed model is demonstrated by the comparison with nine state-of-the-art clustering methods on eight datasets.
    The code is available at \url{https://github.com/lwl-learning/LSDGSymNMF}.
\end{abstract}

\begin{IEEEkeywords}
    Adaptive similarity, symmetric nonnegative matrix factorization, orthogonality regularization, clustering
\end{IEEEkeywords}

\section{Introduction}
\IEEEPARstart{N}{onegative} matrix factorization (NMF) \cite{Lee1999LearningTP, Lee2000AlgorithmsFN} is a well-known dimensionality reduction method, which decomposes a non-negative data matrix $X \in \R_{+}^{m \times n}$ into two non-negative low-rank matrices $U \in \R_{+}^{m \times r}$ and $V \in \R_{+}^{n \times r}$ by solving the following problem:
\begin{equation}
    \min_{U, V \geq 0} \|X - U V^T \|_F^2,
\end{equation}
where $m$ and $n$ are the dimension and the number of data points, respectively, $r$ is the user-specified factorization rank, and $\|X \|_F = \sqrt{\sum_{i, j} x_{ij}^2}$ is the Frobenius norm of a matrix.
NMF has been successfully applied in document clustering \cite{Shahnaz2006DocumentCU, Yang2005DocumentCB}, facial feature extraction \cite{Lee1999LearningTP} and community detection \cite{Wu2018NonnegativeMF}.
For a comprehensive review of NMF, see \cite{Wang2013NonnegativeMF}.

Symmetric NMF (SymNMF) is a variant of NMF, which factors a pairwise similarity matrix $S \in \R^{n \times n}$ into the form $VV^T$ for $V \in \R^{n \times r}$, i.e.,
\begin{equation}
    \min_{V \geq 0} \|S - VV^T \|_F^2,
    \label{eq:SymNMF_intro}
\end{equation}
where $r$ is the predefined number of clustering classes.
SymNMF can be viewed as a relaxation of $k$-means clustering and spectral clustering \cite{Kuang2015SymNMF, Ding2005Onthe}, and has the ability of clustering linearly non-separable data \cite{Long2007RelationalCB, Kuang2012SymmetricNM, Wu2018PairwiseCP}.
In SymNMF, $S$ is usually constructed by $k$-nearest neighbor ($k$-NN) graph as follows:
\begin{equation}
s_{ij} = \begin{cases}
\kappa(x_i, x_j), & \text{if} \; x_j \; \text{is a} \; k \text{-NN of} \; x_i \\
0, & \text{otherwise}
\end{cases},
\label{eq:S_knn}
\end{equation}
where $x_i \in \R^{m}$ denotes the $i$-th data point, $\kappa(x_i, x_j)$ is a kernel function, e.g., $\kappa(x_i, x_j) = \exp(-\|x_i - x_j \|_2^2 / \sigma^2)$ for the Gaussian kernel.
The $k$-NN graph $S$ can be expressed as the sum of the first $k$-NN slices, i.e., $S = \sum_{i=1}^{k} A^{(i)}$, where the $k$-th NN slice $A^{(k)} \in \R^{n \times n}$ is defined as:
\begin{equation}
a_{ij}^{(k)} = \begin{cases}
\kappa(x_i, x_j), & \text{if} \; x_j \; \text{is a} \; k \text{-th NN of} \; x_i \\
0, & \text{otherwise}
\end{cases}.
\end{equation}
Each $A^{(k)}, k = 1, \dots, n$ contains $n$ neighbor relations.
We name the proportion of the correct neighbor relations as the correct rate, where the correct neighbor relation means two neighboring samples $x_i$ and $x_j$ are with same ground-truth class.

\begin{figure}[t]
    \centering
    \subfloat[]{\includegraphics[width=1.65in]{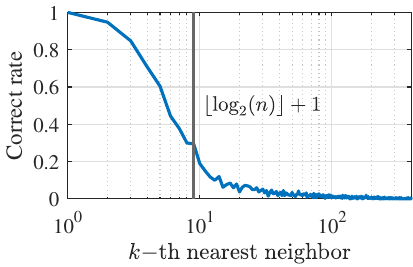} %
        \label{correct_k}}
    \subfloat[]{\includegraphics[width=1.65in]{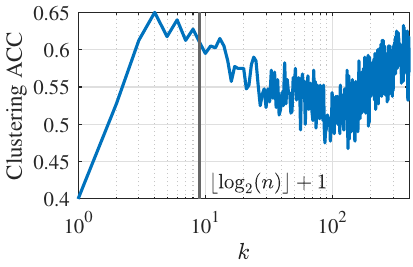} %
        \label{ACC_k}}
    \caption{(a) Correct rate of each $k$-th NN slice $A^{(k)}$.
    (b) Clustering ACC trained by standard SymNMF \cite{Kuang2015SymNMF} on the ORL dataset with respect to $k$, where the kernel function $\kappa(x_i, x_j)$ is defined by the self-tuning method \cite{ZelnikManor2004SelfTuningSC}. 
    }
\end{figure}

By observating the correct rate of each $A^{(k)}$ on the ORL dataset, which is shown in Fig. \ref{correct_k}, we found that only the first few $A^{(k)}$s have high ACC, and the rest contain a lot of error relations.
As a result, the clustering may be misled when the parameter $k$ in \eqref{eq:S_knn} is large. 
For example, we trained the standard SymNMF \cite{Kuang2015SymNMF} on the ORL dataset, and the clustering ACC with respect to $k$ is shown in Fig. \ref{ACC_k}.
It can be seen that the value of $k$ has a significant impact on ACC, and high ACC usually occurs when $k$ is small.
For this reason, Luxburg \cite{Luxburg2004ATO} suggested choosing $k = \lfloor \log_2(n) \rfloor + 1$, which is widely used in SymNMF \cite{Kuang2015SymNMF, Wu2018PairwiseCP, Jia2021SemiAS}.
However, this simple strategy loses a lot of similarity information.

In this paper, we propose a more general and powerful approach to fully utilize the similarity information and reduce the impact of misleading relations on clustering, which is based on the following assumption:

\noindent \textbf{Assumption 1: the $k$-th NN slice with a higher correct rate should be considered more important because they are more reliable.}

To model this importance, we construct $S$ as a weighted $k$-NN graph as follows:
\begin{equation}
    S(w) = \sum_{k=1}^{n} w_k A^{(k)}, \quad \text{where} \quad w \geq 0, w^T 1_n = 1.
    \label{eq:S_w_intro}
\end{equation}
In \eqref{eq:S_w_intro}, $1_n$ is an $n$-dimensional vector full of $1$s, $w_k$ is a learnable non-negative coefficient reflecting the importance of $A^{(k)}$.
The larger $w_k$ is, the more important the $k$-th NN is, and $w_k=0$ means that the $k$-th NN is unselected.

A similar idea was proposed in \cite{nie2014clustering}, which learns an adaptive $k$ neighbors graph (not necessarily $k$-nearest neighbors) by exploring the local connectivity of data.
Huang et. al. \cite{Huang2017NonnegativeMF} further developed this approach in NMF, named NMFAN, whose neighbors are adaptively learned by combining local connectivity in both data and feature.
Experiments have shown that this method is more competitive than $k$-NN \cite{nie2014clustering, Huang2017NonnegativeMF}.

Another approach is to adaptively learn a better graph $S$ from an initial low-quality graph $S_0$ (e.g. $k$-NN graph).
Typically, $S$ is assumed to have some desirable properties, which are described by some constraints or regularizations.
For example, bi-stochastic structure \cite{Wang2022RobustBS}, self-expressive property \cite{Jia2020ClusteringAG}, and consistent with the available supervisory information \cite{Wu2018PairwiseCP, Jia2021SemisupervisedAS, Jia2020PairwiseCP, Jia2024SemisupervisedSM}.
In this approach, the learned $S$ heavily relies on the initial $S_0$.
Please see the detailed discussions in Section \ref{sec:2_2}.

A common \textbf{challenge} of the above two approaches is the high dimensionality of search space of $S$, i.e., $n^2$, because every $s_{ij}$ is a free variable. 
Even though several constrains, such as $S = S^T, \mathrm{diag}(S) = 0, S 1_n = 1_n$ can reduce the freedom of the the variables, it still remains an $\mathcal{O}(n^2)$ dimension space.
Therefore, it is difficult to learn a high-quality similarity matrix in such a high dimensional space.

Differently, in our approach, we construct $S(w)$ in an $n - 1$ dimensional space spanned by each $A^{(k)}, k = 1, \dots, n$.
This low dimensionality makes it much easier to learn the optimal parameter $w^{\ast}$.
Since the underlying optimal $w^{\ast}$ is unknown, we must capture the properties that $w^{\ast}$ has to reduce the search space.
Based on the Assumption 1, $w^{\ast}$ should consistent with the correct rate curve shown in Fig. \ref{correct_k}.
To validate this, we test a simple model that combines \eqref{eq:S_w_intro} with \eqref{eq:SymNMF_intro} as follows:
\begin{equation}
    \min_{V, w \geq 0} \frac{1}{2} \|S(w) - VV^T \|_F^2 + \frac{\mu}{2} \|S(w) \|_F^2, \; \text{s.t.} \; w^T 1_n = 1,
    \label{eq:toy_model}
\end{equation}
where $\mu$ is a hyper-parameter that controls the density of $S(w)$.
With proper $\mu$, we trained \eqref{eq:toy_model} on the ORL dataset, the result in shown in Fig. \ref{fig:toy_model_res}.
\begin{figure}[t]
    \centering
    \includegraphics[width=0.8\linewidth]{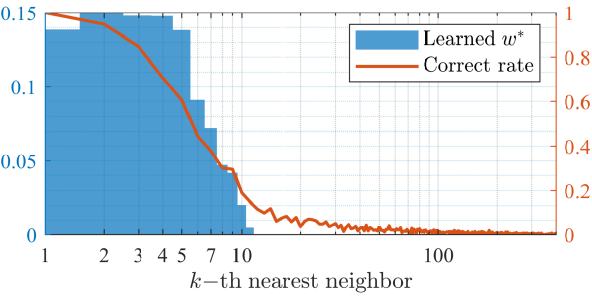}
    \caption{The learned $w^{\ast}$ in \eqref{eq:toy_model} (left $y$-axis) and the correct rate (right $y$-axis). It can be seen that $w^{\ast}$ is sparse and consistent with correct rate.}
    \label{fig:toy_model_res}
\end{figure}
It can be seen that $w^{\ast}$ is consistent with correct rate.
More detailed results on \eqref{eq:toy_model} can be found in Section \ref{sec:ablation}.

To further improve the performance of \eqref{eq:toy_model}, we assume that $w^{\ast}$ should satisfy another property, captured by the following assumption:

\noindent\textbf{Assumption 2: The ideal $S(w^{\ast})$ should correspond to a discriminative clustering result.}

In SymNMF \eqref{eq:SymNMF_intro} and its variant \eqref{eq:toy_model}, the low-rank matrix $V$ reflects the clustering result.
Accordingly, $VV^T$ can be seen as a similarity matrix.
The larger $(VV^T)_{ij}$ is, the more likely $x_i$ and $x_j$ are belonging to the same class.
However, as shown in Fig. \ref{fig:toy_model_res}, $w^{\ast}$ (and $S(w^{\ast})$) is highly sparse.
Does $s_{ij}(w) = 0$ indicate that we consider $x_i$ and $x_j$ to be dissimilar, or we just not know their similarity?
For the former, $(VV^T)_{ij}$ should be near zero; for the latter, $(VV^T)_{ij}$ should not be affected.
However, the model \eqref{eq:toy_model} cannot distinguish between these two cases.

To address this problem, several works \cite{ Wu2021PositiveAN, Jia2021SemiAS} introduce a dissimilarity matrix $D$ that with dual structure of $S$.
When $s_{ij} = 0$, $d_{ij} > 0$ and $d_{ij} = 0$ can distinguish the above two cases.
By simultaneously learning $S$ and $D$, the discriminative of $S$ is improved.
One limitation of this approach is that additional supervision information is required to construct $D$.

In this paper, we notice that many of the later $A^{(k)}$s in Fig. \ref{fig:toy_model_res} have low  correct rate, which can well reflect the dissimilarity information.
Inspired by this observation, we think $D$ can be constructed using a $k$-farthest neighbor graph.
Similar to $S$, we construct $D$ as a weighted $k$-NN graph with another coefficient vector $p \in \R^{n}$, i.e.:
\begin{equation}
    D(p) = \sum_{k=1}^{n} p_k A^{(k)}, \quad \text{where} \quad p \geq 0, p^T 1_n = 1.
    \label{eq:D_p_intro}
\end{equation}

The discriminability of clustering results is also closely related to the orthogonality of $V$.
$V$ can be seen as a soft clustering result, i.e., in $i$-th row of $V$, the column index corresponding to maximum element is the class index of $x_i$.
When $V$ is column orthogonal, i.e., $\forall i \ne j, v_i^T v_j = 0$, each row of $V$ has only one non-zero element, which corresponds to hard clustering result.
Therefore, the orthogonality constraint of $V$ is introduced in \cite{Paul2016OrthogonalSN}.
Moreover, orthogonal NMF (ONMF) has been proven to be equivalent to weighted spherical $k$-means clustering \cite{Pompili2014TwoAF}.
Inspired by this interpretability, many ONMF methods \cite{Pompili2014TwoAF, Ding2006OrthogonalNM, Wang2021ClusteringBO} have been proposed to make $V$ more discriminative.
The key \textbf{challenge} in these methods is dealing with the non-convex orthogonality constraint $V^T V = I$.
To this end, considering that $v_i^T v_i = 1$ is not essential, Li et. al \cite{Li2015TwoEA} modified the constraint to a regularization $\sum_{i=1}^{r} \sum_{j\ne i} v_i^T v_j$.
However, this formulation may lead to incorrect clustering when a column in $V$ is all zeros, resulting the number of obtained clusters smaller than $r$. 
See the discussion in \ref{sec:2_3} for details.
The negative log-determinant function can also be used as an orthogonality regularization \cite{Liu2017LargeCN}, but its optimization algorithm have no theoretical guarantee of convergence.

In this paper, a new form of orthogonality regularization is proposed.
By column-wisely updating, theoretical convergence is proven under the progressive hierarchical alternating least squares (PHALS) \cite{Hou2023APrgressiveHA} framework.

The contributions of this paper are three-folds:
\begin{enumerate}
    \item A new approach for similarity and dissimilarity learning is proposed, which can be seen as a weighted $k$-NN graph.
    Under this approach, the similarity relations are fully utilized, and the dimension of the search space is reduced from $\mathcal{O}(n^2)$ to $n - 1$.
    \item A new form of orthogonality regularization is proposed, with discussions on its geometric interpretation and numerical stability.
    By column-wisely updating, theoretical convergence is proven under the PHALS \cite{Hou2023APrgressiveHA} framework.
    \item An alternative optimization algorithm is designed to solve the proposed model with $\mathcal{O}(n^2 r)$ complexity per iteration. 
    With theoretical proof, the variables converge to a stationary point that satisfied the KKT conditions.
\end{enumerate}

\section{Related Work} \label{sec:2}
\subsection{Adaptive Similarity Methods} \label{sec:2_2}
To adaptively learn a better graph from an initial low-quality graph, Jia et. al \cite{Jia2020ClusteringAG} proposed a self-expressive aware graph construction model as follows:
\begin{equation}
\begin{aligned}
    &\min_{S, V \geq 0} \|S - VV^T \|_F^2 + \alpha \|X - XS \|_F^2 + \beta \|S - S_0 \|_F^2 \\
    &\text{s.t.}, \quad \mathrm{diag}(S) = 0,
\end{aligned}
\end{equation}
where $S_0$ is the initial graph constructed by \eqref{eq:S_knn}, $\mathrm{diag}(S) = 0$ means that the diagonal elements of $S$ are all zeros.

By leveraging ensemble clustering, S$^3$NMF \cite{Jia2022S3NMF} was proposed to boost clustering performance progressively, which is formulated as:
\begin{equation}
    \min_{S, V_i, \alpha \geq 0} \sum_{i=1}^{c} \alpha_i^{\tau} \|S - V_i V_i^T \|_F^2, \quad  \text{s.t.}, \quad \alpha^T 1_n = 1,
\end{equation}
where $V_i \in \R^{n \times r}$ is the $i$-th factor matrix, $\alpha$ is the weigh vector balancing the contribution of each $V_i$, $\tau$ is the paramater that controls the distribution of $\alpha$ (sharp or uniform).

Moreover, by utilizing supervisory information, Jia et. al \cite{Jia2021SemiAS} induced a dissimilarity matrix to enhance similarity learning, which is formulated as:
\begin{equation}
\begin{aligned}
    &\min_{S, D, V \geq 0} \|S - VV^T \|_F^2 + \eta \langle S, D \rangle \\
    &\qquad \qquad + \alpha \mathrm{Tr}(D L_{S_0} D^T) + \beta \mathrm{Tr}(S L_{S_0} S^T) \\
    & \text{s.t.}, \quad s_{ij}=1, d_{ij}=0, \; \forall (i, j) \in \mathcal{M}, \\
    &\qquad \;\;\, s_{ij}=0, d_{ij}=1, \; \forall (i, j) \in \mathcal{C},
\end{aligned}
\end{equation}
where $L_{S_0} = \mathrm{Diag}(S_0 1_n) - S_0$ is the graph laplacian of $S_0$, $\mathrm{Diag}(x)$ is a diagonal matrix with $x$ as its diagonal elements.
$\mathcal{M}$ and $\mathcal{C}$ are must-link set and cannot-link set, respectively, provided by supervisory information.

In this approach, the learned $S$ heavily relies on the initial $S_0$, because $S_0$ either acts as the initialization of $S$ \cite{Jia2022S3NMF} or been incorporated by a regularization like $\|S - S_0 \|_F^2$ \cite{Jia2020ClusteringAG} and $\mathrm{Tr}(S L_{S_0} S^T)$ \cite{Jia2021SemiAS}.
Conversely, NMFAN \cite{Huang2017NonnegativeMF} learns an adaptive $k$ neighbors (not necessarily $k$-nearest neighbors) by combining local connectivity in both data and feature spaces, which is formulated as:
\begin{equation}
\begin{aligned}
    &\min_{S, U, V \geq 0} \|X - UV^T \|_F^2 + \lambda \mathrm{Tr}(V^T L_S V) \\
    &\qquad \qquad + \mu \left( \mathrm{Tr}(X^T L_S X) + \gamma \|S \|_F^2 \right) \\
    &\text{s.t.}, \quad S^T 1_n = 1_n.
\end{aligned}
\end{equation}

A common challenge of the above methods is the $\mathcal{O}(n^2)$ dimension of search space of $S$, because every $s_{ij}$ is a free variable. 
Therefore, it is difficult to learn a high-quality similarity matrix in such a high dimensional space.
As a comparison, we construct $S$ in an $n - 1$ dimensional space spanned by each $k$-th NN slice $A^{(k)}$, which makes it easier to obtain the optimal similarity in a low-dimensional space.

\subsection{Orthogonality Regularization} \label{sec:2_3}
The orthogonality constraint of $V$ is introduced in \cite{Paul2016OrthogonalSN} to enhence the discriminability of clustering result.
Additionally, orthogonal NMF has been proven to be equivalent to weighted spherical $k$-means clustering \cite{Pompili2014TwoAF}.
Inspired by this interpretability, some orthogonality reglarizations are proposed to make $V$ more discriminative.
For example, the approximately orthogonal NMF \cite{Li2015TwoEA} is formulated as
\begin{equation}
\begin{aligned}
    &\min_{U, V \geq 0} \tfrac{1}{2} \|X - UV^T \|_F^2 - \tfrac{\lambda}{2} \mathcal{R}_{\text{off-diag}}(V) \\
    & \textstyle \text{where} \quad \mathcal{R}_{\text{off-diag}}(V) = -\sum_{i=1}^{r} \sum_{j \ne i} v_i^T v_j.
    \label{eq:AONMF}
\end{aligned}
\end{equation}
The $\mathcal{R}_{\text{off-diag}}(V)$ is maximized when $V$ is column-orthogonal, i.e., $\forall i \ne j, v_i^T v_j = 0$.
Additionally, the $\mathcal{R}_{\text{off-diag}}(V)$ is also derivated in \cite{Liu2017LargeCN} by geometrically maximizing the pairwise angles between $v_i$ and $v_j$.
An optimization algorithm of \eqref{eq:AONMF} that with convergence guarantee can be seen in \cite{Li2020TwoFV}.
However, maximize $\mathcal{R}_{\text{off-diag}}(V)$ not only makes $V$ more orthogonal, but also reduces the scale of $V$, which may lead to incorrect clustering when a column in $V$ is an all-zeros vector.

Another form of orthogonality regularization is to maximize the following log-determinant function:
\begin{equation}
    \mathcal{R}_{\text{log-det}}(V) = \log\det(V^T V).
\end{equation}
Geometrically, $\det(V^T V)$ is equal to the square of the volume of the parallelotope formed by the vectors $\{v_1, \dots, v_r \}$ \cite{Liu2017LargeCN}.
Moreover, the $\mathcal{R}_{\text{log-det}}(V)$ can also be written as:
\begin{equation}
    \log\det(V^T V) = \log \prod_{i=1}^{r} \lambda_i(V^T V) = 2\sum_{i=1}^{r} \log \sigma_i (V),
\end{equation}
where $\lambda_i(X)$ and $\sigma_i(X)$ are the $i$-th largest eigenvalue and singular value of $X$, respectively.
Once a column in $V$ is $0_n$, then $\sigma_r(V) = 0$ and $\mathcal{R}_{\text{log-det}}(V) = +\infty$.
Therefore, the disadvantage of $\mathcal{R}_{\text{off-diag}}(V)$ can be avoid.
However, the optimization algorithm proposed by \cite{Liu2017LargeCN} has no theoretical guarantee of convergence.

In this paper, instead of updating $V$ as a whole, we modified the function $\log\det(V^T V)$ into the form $v_j^T M v_j$, and updating each $v_j$ sequentially.
Due to its simplicity, the theoretical convergence is proven under the framework of progressive hierarchical alternating least squares (PHALS) \cite{Hou2023APrgressiveHA}.
This modification makes the regularization more practical and broadly applicable in NMF and SymNMF.

\section{Proposed Model} \label{sec:3}
\subsection{Similarity and Dissimilarity Learning}
As mentioned earlier, the similarity matrix $S$ is usually constructed using a $k$-NN graph, which may mislead clustering since the NNs might belong to different clusters.
Additionally, existing adaptive similarity graph learning methods suffer from the high-dimensional search space and dependence on a low-quality initial similarity matrix $S_0$.
To address these issues, we propose constructing $S$ in an $n-1$ dimensonal space spanned by each $k$-th NN slice of the affinity matrix.
Specifically, let $A^{(k)} \in \R^{n \times n}$ represent the $k$-th NN slice as follows:
\begin{equation}
    a_{ij}^{(k)} = \begin{cases}
    \kappa(x_i, x_j), & \text{if} \; x_j \; \text{is a} \; k \text{-th NN of} \; x_i \\
    0, & \text{otherwise}
\end{cases}.
\label{eq:Ak}
\end{equation}
Then, $S$ is constructed as a weighted $k$-NN graph as follows:
\begin{equation}
    S(w) = \sum_{k=1}^{n} w_k A^{(k)}, \quad \text{where} \quad w \geq 0, w^T 1_n = 1.
    \label{eq:S_w}
\end{equation}
In \eqref{eq:S_w}, $w \in \R^{n}$ is the nonnegative combination coefficient that reflects the importance of $A^{(k)}, k = 1,\dots, n$, and we normalize $w$ so that $w^T 1_n = 1$.
Due to this constraint, the dimension of $w$ and $S(w)$ is $n-1$.
Accordingly, compared with existing adaptive similarity methods that search for $S$ in an $\mathcal{O}(n^2)$ dimensional space (as discussed in \ref{sec:2_2}), this approach is more feasible to reach the optimal similarity $S(w^{\ast})$.
Moreover, the $k$-NN graph $S_0$ defined in $\eqref{eq:S_knn}$ is a special case of \eqref{eq:S_w} that can be denoted as $S(w_0)$, where $w_0(1), \dots, w_0(k) = 1/k$ and $w_0(k+1), \dots, w_0(n) = 0$, i.e.:
\begin{equation}
    \frac{1}{k} S_{0} = S(w_0) = \frac{1}{k} \sum_{i=1}^{k} A^{(i)}.
\end{equation}
Therefore, compared with existing methods that heavily rely on the initial $S_0$, the proposed method is more stable since it leverages the entire affinity matrix by different slices $A^{(1)}, \dots, A^{(n)}$.

When $k$ is large, $\|A^{(k)} \|_F$ becomes very small.
To balance this, we normalize $A^{(k)}$ after \eqref{eq:Ak} as follows:
\begin{equation}
    \bar{A}^{(k)} \leftarrow A^{(k)} / \|A^{(k)} \|_F,
\end{equation}
and calculate $S(w)$ in \eqref{eq:S_w} using $\bar{A}^{(k)}$.
In the rest of this paper, we will continue to use $A^{(k)}$ for simplicity in notation.

Based on these definitions, we propose the similarity learning model as
\begin{equation}
\begin{aligned}
    \min_{V, w \geq 0} \tfrac{1}{2} \|S(w) - VV^T \|_F^2 - \alpha \mathcal{R}(V),  \; \text{s.t.}, \; w^T 1_n = 1.
    \label{eq:model_1}
\end{aligned}
\end{equation}
In this model, $S$ in \eqref{eq:SymNMF_intro} is replaced with $S(w)$. 
The term $\mathcal{R}(V)$ represents the orthogonality regularization, which will be introduced later.
The hyper-parameter $\alpha \geq 0$ controls the contribution of the orthogonality regularization.

To improve the discriminative of $S$, we introduce a dissimilarity matrix $D$ that with dual structure of $S$, which is constructed as a weighted $k$-NN graph with another coefficient $p \in \R^{n}$ as follows:
\begin{equation}
    D(p) = \sum_{k=1}^{n} p_k A^{(k)}, \quad \text{where} \quad p \geq 0, p^T 1_n = 1.
    \label{eq:D_p}
\end{equation}
Then, the joint similarity and dissimilarity learning model is formulated as:
\begin{equation}
\begin{aligned}
    &\min_{V, w, p \geq 0} \tfrac{1}{2} \|S(w) - VV^T \|_F^2 \!+\! \beta \langle D(p), VV^T \rangle \!-\! \alpha \mathcal{R}(V) \\ 
    &\text{s.t.} \quad w^T 1_n = p^T 1_n = 1, w^T p = 0,
\end{aligned}
\end{equation}
where $\langle D(p), VV^T \rangle = \sum_{i=1}^{n} \sum_{j=1}^{n} d_{ij}(p) \cdot (VV^T)_{ij}$ represents the dissimilarity regularization, $\beta \geq 0$ controls its contribution.
The mechanism of the dissimilarity regularization is as follows:
when $d_{ij}(p)$ is large, the $x_i$ and $x_j$ are considered to be dissimilar, thus the $(VV^T)_{ij}$ should be close to zero;
when $d_{ij}(p)$ is near zero, the $(VV^T)_{ij}$ is unimpacted.
Moreover, the constraint $w^T p = \sum_{k=1}^{n} w_k \cdot p_k = 0$ ensures that all $A^{(k)}$s cannot be used to construct both similarity and dissimilarity matrices simultaneously.

Finally, our proposed model is formulated as:
\begin{equation}
\begin{aligned}
    &\min_{V, w, p \geq 0} \tfrac{1}{2} \|S(w) - VV^T \|_F^2 \!+\! \beta \langle D(p), VV^T \rangle \!-\! \alpha \mathcal{R}(V) \\ 
    &\qquad \qquad  + \tfrac{\mu - 1}{2} \| S(w) \|_F^2 + \tfrac{\mu}{2} \|D(p) \|_F^2 \\
    &\text{s.t.} \quad w^T 1_n = p^T 1_n = 1, w^T p = 0,
    \label{eq:model}
\end{aligned}
\end{equation}
where $\frac{\mu - 1}{2} \| S(w) \|_F^2$ and $\frac{\mu}{2} \|D(p) \|_F^2$ control the densities of $S(w)$ and $D(p)$, respectively.
The larger the $\mu$ is, the more $w_k$ and $p_k$ are activated. 
Since the first term of \eqref{eq:model} already includes $\frac{1}{2} \|S(w) \|_F^2$, the coefficient of the fourth term becomes $\frac{\mu - 1}{2}$.

After solving \eqref{eq:model}, to get clustering result, we construct an augmented similarity matrix $Z \in \R^{n \times n}$ by combining the learned $S, D$ and $V$.
Let $Y = VV^T$, and we normalize $S, D, Y$ into range $[0, 1]$, i.e., $S \leftarrow S / \max(S)$.
Then, let
\begin{equation}
    z_{ij} = \begin{cases}
        1 - (1 - y_{ij} + d_{ij})(1 - s_{ij}), & \text{if} \; y_{ij} \geq d_{ij}\\
        (1 + y_{ij} - d_{ij})s_{ij}, & \text{if} \; y_{ij} < d_{ij} \\
    \end{cases}.
    \label{eq:aug_affinity}
\end{equation}
When $y_{ij} \geq d_{ij}$, $x_i$ and $x_j$ are likely to belong to the same class, \eqref{eq:aug_affinity} will increase the corresponding similarity $s_{ij}$.
Similarly, when $y_{ij} < d_{ij}$, $s_{ij}$ will be depressed.
Therefore, the similarity $S$ is further enhanced.
Finally, we apply spectral clustering \cite{AY2002OnSC} on $Z$ to get clustering result.

\subsection{Analysis of the Model in \eqref{eq:model}}
To better understand how our model works, we can view \eqref{eq:model} from the perspective of $V$ and $(w, p)$ respectively.
Before that, we first point out that $A^{(k)}$, $S(w)$ and $D(p)$ have the following useful properties:
\begin{itemize}
    \item $\forall k, \|A^{(k)} \|_F = 1$ and $\forall k \ne t, \langle A^{(k)}, A^{(t)} \rangle = 0$;
    \item $\|S(w) \|_F^2 = \sum_{k=1}^{n} w_k^2$ and $\|D(p) \|_F^2 = \sum_{k=1}^{n} p_k^2$;
    \item $\langle S(w), D(p) \rangle = \sum_{k=1}^{n} w_k \cdot p_k$.
\end{itemize}

\textbf{From the perspective of $V$}, \eqref{eq:model} is equivalent to
\begin{equation}
    \min_{V \geq 0} \frac{1}{2} \|S(w) - \beta D(p) - VV^T\|_F^2 - \alpha \mathcal{R}(V).
    \label{eq:model_V}
\end{equation}
When focusing on the first term, there are four cases to discuss:
\begin{enumerate}
    \item $s_{ij}(w) > 0$ and $d_{ij}(p) = 0$, which means that $x_i$ is similar to $x_j$, leading to $(VV^T)_{ij} > 0$
    \item $s_{ij}(w) = 0$ and $d_{ij}(p) > 0$, which means that $x_i$ is dissimilar to $x_j$, leading to $(VV^T)_{ij} \approx 0$
    \item $s_{ij}(w) = 0$ and $d_{ij}(p) = 0$, which means that the relation between $x_i$ and $x_j$ is unknown.
    One might think that this case would lead to $(VV^T)_{ij} \approx 0$.
    However, $VV^T$ is a low-rank matrix, thus the impact of this case on $V$ is negligible compared to the second case.
    \item $s_{ij}(w) > 0$ and $d_{ij}(p) > 0$ never occur since the constraint $w^T p = \langle S(w), D(p) \rangle = 0$.
\end{enumerate}
The role of $\mathcal{R}(V)$ will be analyzed in the next subsection.

\textbf{From the perspective of $(w, p)$}, \eqref{eq:model} is equivalent to
\begin{equation}
\begin{aligned}
    &\min_{w, p \geq 0} \frac{\mu}{2}\sum_{k=1}^{n} \left(w_k^2 + p_k^2 \right) + \sum_{k=1}^{n} c_k \cdot \left(\beta p_k - w_k \right) \\
    &\text{s.t.} \quad w^T 1_n = p^T 1_n = 1, w^T p = 0,
\end{aligned}
\end{equation}
where $c_k = \left\langle A^{(k)}, VV^T \right\rangle$.
It can be seen that $\mu$ controls the density of $w$ and $p$ (as well as $S(w)$ and $D(p)$, respectively), while $\beta$ controls the contribution of dissimilarity regularization.
Generally, the larger the $k$, the smaller the $c_k$.
For example, when $V$ is the ground-truth class assignment matrix, $c_k$ is equivalent to the accuracy shown in Fig. \ref{correct_k}.
This phenomenon causes $w$ and $p$ to automatically learn $k$ nearest neighbors and $t$ farthest neighbors, for some $k$ and $t$, respectively, with the weights being consistent with the correct rate of $A^{(k)}$s.

\subsection{New Orthogonality Regularization}
\begin{figure}[!t]
    \centering
    \includegraphics[width=\linewidth]{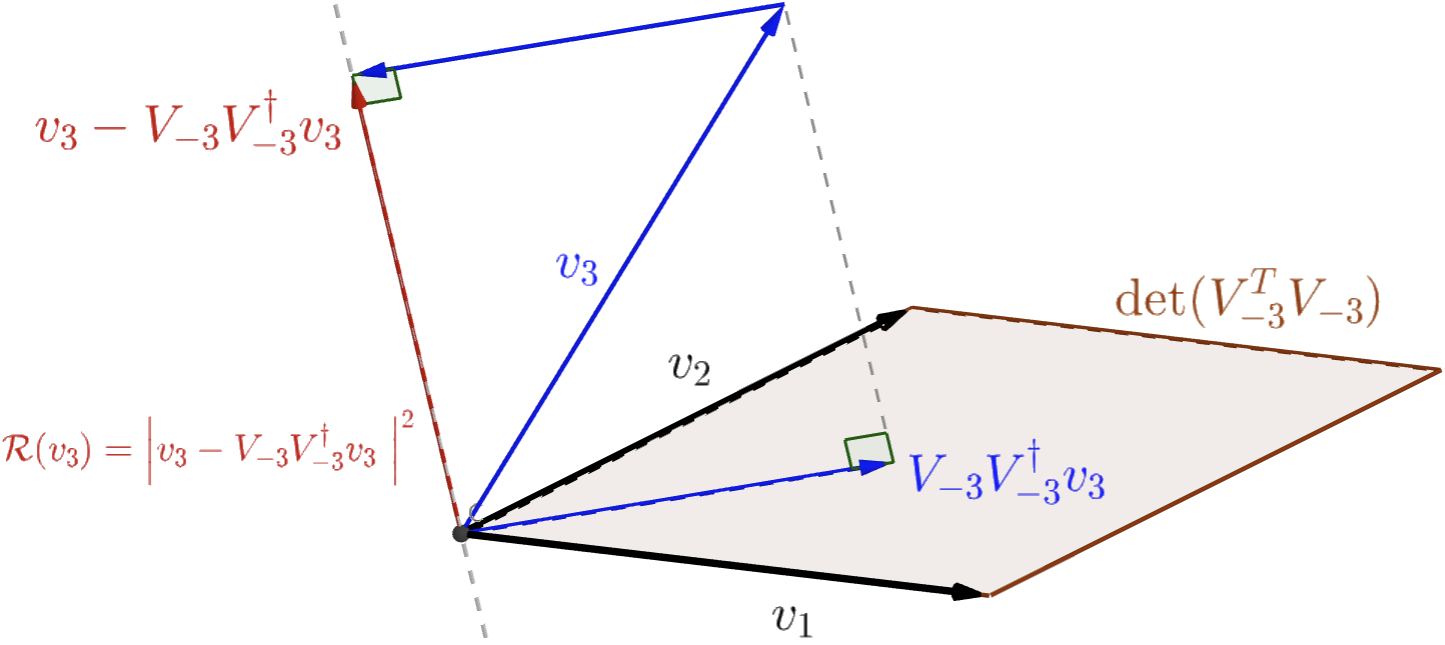}
    \caption{The geometric meaning of $\mathcal{R}(v_3)$, which can be seen as the square of distance between $v_3$ and the plane spanned by $\{v_1, v_2 \}$.}
    \label{fig:orth_geo}
\end{figure}

The role of orthogonality regularization $\mathcal{R}(V)$ is to make the learned $S(w)$ more discriminative.
As dicussed previously in subsection \ref{sec:2_3}, the widely used orthogonality regularization $\mathcal{R}_{\text{off-diag}}(v) = -\sum_{i=1}^{r} \sum_{j \ne i} v_i^T v_j$ \cite{Li2015TwoEA} may lead to incorrect clustering when a column in $V$ is $0_n$, and $\mathcal{R}_{\text{log-det}}(v) = \log\det(V^T V)$ \cite{Liu2017LargeCN} has no optimization algorithm that guarantees convergence.
To address this problem, we modified the function $\log\det(V^T V)$ to column-wisely updating $V$.
Specifically, the following proposition states that $\det(V^T V)$ can be reformulated into a function w.r.t. $v_j$.
\begin{proposition} \label{proposition_1}
    Let $V_{-j} = \left[v_1, \cdots, v_{j-1}, v_{j+1}, \cdots, v_{r} \right]$, $I_n$ represents the identity matrix of size $n$, $V^{\dag} \in \R^{r \times n}$ represents the Moore-Penrose pseudo-inverse of $V$.
    Assuming that $V \in \R^{n \times r}$ and $\mathrm{rank}(V) = r$, then $\forall j = 1,\dots, r$, we have:
    \begin{equation}
        \det(V^T V) = \det \left(V_{-j}^T V_{-j} \right) v_j^T \left(I_n - V_{-j} V_{-j}^{\dag} \right) v_j.
        \label{eq:det_vj}
    \end{equation}
\end{proposition}
\begin{proof} 
    By spilting $V$ as $V = [V_{1:j-1}, v_j, V_{j+1:r}]$, we have:
    \begin{equation}
    \resizebox{0.89\hsize}{!}{$
    \begin{aligned}
        \det(V^T V) 
        &= \begin{vmatrix}
            V_{1:j-1}^{T} V_{1:j-1}^{T} & V_{1:j-1}^{T} v_j & V_{1:j-1}^{T} V_{j+1:r} \vspace{2mm} \\ 
            v_j^{T} V_{1:j-1}^{T} & v_j^{T} v_j & v_j^{T} V_{j+1:r} \vspace{2mm} \\
            V_{j+1:r}^{T} V_{1:j-1}^{T} & V_{j+1:r}^{T} v_j & V_{j+1:r}^{T} V_{j+1:r}
        \end{vmatrix} \\
        &= \begin{vmatrix}
            V_{-j}^{T} V_{-j} & V_{-j}^T v_j  \vspace{2mm} \\
            v_j^{T} V_{-j} & v_j^T v_j
        \end{vmatrix}.
    \end{aligned}
    $}
    \end{equation}
    By appling $\det \begin{psmallmatrix} A & B \\ C & D \end{psmallmatrix} = \det(D) \det(A - B D^{-1} C)$, we have:
    \begin{equation}
        \begin{vmatrix}
            V_{-j}^{T} V_{-j} & V_{-j}^T v_j  \vspace{2mm} \\
            v_j^{T} V_{-j} & v_j^T v_j
        \end{vmatrix}
         = v_j^T v_j \left| V_{-j}^T V_{-j} - \frac{V_{-j}^T v_j v_j^T V_{-j}}{v_j^T v_j} \right|
         \label{det_derivation_1}.
    \end{equation}
    According to the Sylvester's determinant theorem: $\det(X + AB) = \det(X) \det(I_n + B X^{-1} A)$, \eqref{det_derivation_1} can be further simpliefied to
    \begin{equation}
    \resizebox{0.89\hsize}{!}{$
    \begin{aligned}
        &(v_j^T v_j) \det(V_{-j}^{T} V_{-j}) \left(1 - \frac{v_j^T \left[V_{-j} \left(V_{-j}^T V_{-j} \right)^{-1} V_{-j}^{T} \right] v_j}{v_j^T v_j} \right) \\
        &= \det(V_{-j}^{T} V_{-j}) \left(v_j^T v_j - v_j^T \left[V_{-j} \left(V_{-j}^T V_{-j} \right)^{-1} V_{-j}^{T} \right] v_j \right) \\
        &= \det(V_{-j}^{T} V_{-j}) \left(v_j^T v_j - v_j^T \left[V_{-j} V_{-j}^{\dag} \right] v_j \right) \\
        &= \det(V_{-j}^{T} V_{-j}) v_j^T\left(I_n - V_{-j} V_{-j}^{\dag} \right)v_j,
    \end{aligned}
    $}
    \end{equation}
    which is the right side of \eqref{eq:det_vj}.
\end{proof}

Geometrically, $\det(V^T V)$ is equal to the square of the volume of the parallelotope formed by the vectors $\{v_1, \dots, v_r\}$ \cite{Liu2017LargeCN}.
Therefore, Proposition \ref{proposition_1} states that the $\det(V^T V)$ can be expressed as the product of $\det(V_{-j}^T V_{-j})$ and the contribution of $v_j$.
Based on this observation, we define the orthogonality regularization w.r.t. $v_j$ as
\begin{equation}
    \mathcal{R}(v_j) = v_j^T \left(I_n - V_{-j} V_{-j}^{\dag} \right) v_j = v_j^T M_j v_j,
\end{equation}
where $M_j = I_n - V_{-j} V_{-j}^{\dag}$.
The geometric meaning of $\mathcal{R}(v_j)$ is illustrated in Fig. \ref{fig:orth_geo}.
The computational merits of $\mathcal{R}(v_j)$ are summarized as follows:
\begin{enumerate}
    \item \textbf{Clustering correctness:} When maximizing  $\mathcal{R}(v_j)$, $v_j$ will not be $0_n$. More generally, $v_j$ will be far from the range space of $V_{-j}$, making the clustering partitions $v_j$ be different from $V_{-j}$.
    \item \textbf{Convergence guarantee:} The algorithm for solving \eqref{eq:model_V} can be derived under the framework of PHALS \cite{Hou2023APrgressiveHA} with theoretical convergence guarantee, which will be introduced in the next section. 
    \item \textbf{Numerical stability:} Since $M_j$ is a positive semi-definite matrix whose largest eigenvalue is one, we have:
    \begin{equation}
        0 \leq \mathcal{R}(v_j) \leq \|v_j \|_2^2 \quad \text{and} \quad 0 \leq \mathcal{R}(V) \leq \|V \|_F^2.
        \label{eq:sup_RV}
    \end{equation}
    This is much more stable than $\mathcal{R}_{\text{log-det}}(V) = \log\det(V^T V)$, whose lower bound is $-\infty$.
\end{enumerate}

\begin{remark}
    $\mathcal{R}(V)$ cannot be explicitly expressed as a function w.r.t. $V$.
    Therefore, optimization problems with $\mathcal{R}(V)$ can only be optimized column-wisely, and $V_{-j}$ is fixed when updating $v_{j}$.
\end{remark}

\section{Optimization} \label{sec:4}
In this section, we propose an alternating iterative algorithm to solve \eqref{eq:model}, i.e., when updating a variable, all the other variables are fixed and treated as constants.
\subsection{Updating $V$}
When fixing $w$ and $p$, the subproblem of \eqref{eq:model} w.r.t. $V$ is:
\begin{equation}
    \min_{V \geq 0} \frac{1}{2} \|S(w) - VV^T\|_F^2 + \beta \langle D(p), VV^T \rangle - \alpha \mathcal{R}(V).
    \label{eq:model_V2}
\end{equation}
For each $j = 1,\dots, r$, $V_{-j}$ is fixed, and the subproblem of \eqref{eq:model_V2} w.r.t. $v_j$ is:
\begin{equation}
\begin{aligned}
    \min_{v_j \geq 0} &\frac{1}{2} \left\|S(w) - V_{-j}V_{-j}^T - v_j v_j^T \right\|_F^2 + \beta \langle D(p), v_j v_j^T \rangle \\
    &\qquad - \alpha v_j^T M_j v_j
\end{aligned},
\end{equation}
which can be further simplified as:
\begin{equation}
    \min_{v_j \geq 0} \frac{1}{2} \left\|S(w) - \beta D(p) + \alpha M_j - V_{-j}V_{-j}^T - v_j v_j^T \right\|_F^2.
    \label{eq:model_vj}
\end{equation}
It can be seen that the subproblem \eqref{eq:model_vj} is mathematically equivalent to rank-one SymNMF, which can be efficiently solved by applying the PHALS algorithm \cite{Hou2023APrgressiveHA}.

\begin{lemma}[Theorem 2, \cite{Hou2023APrgressiveHA}] \label{lemma2}
    When the PHALS algorithm is applied to solving \eqref{eq:model_vj}, the objective functions of \eqref{eq:model_vj} and \eqref{eq:model_V2} are monotonically decreasing.
\end{lemma}

\begin{remark}
    In \eqref{eq:model_vj}, $S(w) - \beta D(p)$ is asymmetric, but it is equivalent to symmetric case because:
    \begin{equation}
    \resizebox{0.89\hsize}{!}{$
        \|M - vv^T\|_F^2 = \|\tfrac{M + M^T}{2} - vv^T \|_F^2 - \|\tfrac{M + M^T}{2} \|_F^2 + \|M \|_F^2.
    $}
    \end{equation}
\end{remark}

\subsection{Updating $(w, p)$}
When fixing $V$, the subproblem of \eqref{eq:model} w.r.t $(w, p)$ is:
\begin{equation}
\begin{aligned}
    &\min_{w, p \geq 0} \frac{\mu}{2}\sum_{k=1}^{n} \left(w_k^2 + p_k^2 \right) + \sum_{k=1}^{n} c_k \cdot \left(\beta p_k - w_k \right) \\
    &\text{s.t.} \quad w^T 1_n = p^T 1_n = 1, w^T p = 0,
\end{aligned}
\label{eq:model_wp1}
\end{equation}
where $c_k = \left\langle A^{(k)}, VV^T \right\rangle$.
Mathematically, \eqref{eq:model_wp1} is a non-convex optimization problem due to the constraint $w^T p = 0$.
To solve this problem, we relax this constraint into a regularization as follows:
\begin{equation}
\resizebox{0.89\hsize}{!}{$
\begin{aligned}
    &\min_{w, p \geq 0} \frac{\mu}{2}\sum_{k=1}^{n} \left(w_k^2 + p_k^2 \right) + \sum_{k=1}^{n} c_k \cdot \left(\beta p_k - w_k \right) + \eta \sum_{k=1}^{n} w_k \cdot p_k \\
    &\text{s.t.} \quad w^T 1_n = p^T 1_n = 1,
\end{aligned}
$}
\label{eq:model_wp}
\end{equation}
where $\eta > 0$ is a hyper-parameter.

\begin{proposition} \label{proposition:opt}
    Let $(w^{\ast}, p^{\ast})$ be a global optimum of \eqref{eq:model_wp}, if it satisfies:
    \begin{equation}
        \sum_{k=1}^{n} w_k^{\ast} \cdot p_k^{\ast} = 0,
        \label{eq:wp_condition}
    \end{equation}
    then $(w^{\ast}, p^{\ast})$ is also a global optimum of \eqref{eq:model_wp1}.
\end{proposition}
\begin{proof}
    Let the objective function and the constraint space of \eqref{eq:model_wp} be $f(w, p, \eta)$ and $\Delta^{n} = \{w \in \R^{n} \;|\; w \geq 0, w^T 1_n = 1 \}$, respectively.
    Since $(w^{\ast}, p^{\ast})$ is a global optimum of \eqref{eq:model_wp}, for any $\bar{\eta} \geq \eta$, we have:
    \begin{equation}
        f(w^{\ast}, p^{\ast}, \eta) \leq f(w, p, \eta) \leq f(w, p, \bar{\eta}), \; \forall \; w, p \in \Delta^{n}.
    \end{equation}
    Moreover, since $\sum_{k=1}^{n} w_k^{\ast} \cdot p_k^{\ast} = 0$, we have $f(w^{\ast}, p^{\ast}, \eta) = f(w^{\ast}, p^{\ast}, \bar{\eta})$, which indicates that:
    \begin{equation}
        f(w^{\ast}, p^{\ast}, \bar{\eta}) \leq f(w, p, \bar{\eta}), \; \forall \; w, p \in \Delta^{n}.
    \end{equation}
    This means that, if we replace $\eta$ in \eqref{eq:model_wp} with $\bar{\eta}$, then $(w^{\ast}, p^{\ast})$ is still a global optimum.
    Particularly, when $\bar{\eta} = +\infty$, \eqref{eq:model_wp} is equivalent to \eqref{eq:model_wp1}, thus the proposition holds.
\end{proof}

Proposition \ref{proposition:opt} inspires us to solve \eqref{eq:model_wp} instead of \eqref{eq:model_wp1}, and the following proposition provides the condition under which \eqref{eq:model_wp} can be efficiently solved.
\begin{proposition}
    If $\mu > \eta$, then \eqref{eq:model_wp} is $(\mu-\eta)$-strongly convex problem, thus \eqref{eq:model_wp} has a unique global optimum.
\end{proposition}
\begin{proof}
    It is easy to verify that the constraint space in \eqref{eq:model_wp} is a convex set.
    Besides, the Hessian of the objective function in \eqref{eq:model_wp} w.r.t. $(w, p)$ is:
    \begin{equation}
        H = \begin{bmatrix}
            \mu I_n & \eta I_n \\
            \eta I_n & \mu I_n
        \end{bmatrix}.
    \end{equation}
    Since $\det \left(\begin{smallmatrix} A & B \\ B & A \end{smallmatrix} \right) = \det(A - B) \det(A + B)$, the eigenvalues of $H$ are values of $\lambda$ that satisfies:
    \begin{equation}
        \det(H - \lambda I_{2n}) = (\mu - \eta - \lambda)^n (\mu + \eta - \lambda)^n = 0,
    \end{equation}
    which indicates that $\lambda = \mu - \eta$ or $\lambda = \mu + \eta$.
    Therefore, the objective function in \eqref{eq:model_wp} is $(\mu - \eta)$-strongly convex, and \eqref{eq:model_wp} has a unique global optimum.
\end{proof}

Benefiting from the strongly convexity, \eqref{eq:model_wp} can be efficiently solved by alternately solving the following two sub-problems at the $t$-th iteration:
\begin{subequations}
    \begin{alignat}{2}
        w^{(t+1)} &=\!\! \mathop{\arg\min}\limits_{w \geq 0, w^T 1_n = 1} \sum_{k=1}^{n} \frac{\mu}{2} w_k^2 + (\eta p_k^{(t)} - c_k) \cdot w_k \\
        p^{(t+1)} &=\!\! \mathop{\arg\min}\limits_{p \geq 0, p^T 1_n = 1}  \sum_{k=1}^{n} \frac{\mu}{2} p_k^2 + (\eta w_k^{(t+1)} + \beta c_k) \cdot p_k,
    \end{alignat}
\end{subequations}
which can be further simplified as:
\begin{subequations}
    \begin{alignat}{2}
        w^{(t+1)} &= \mathop{\arg\min}\limits_{w \geq 0, w^T 1_n = 1} \left\|w + (\eta p^{(t)} - c) / \mu \right\|_2^2 \label{eq:min_w} \\
        p^{(t+1)} &= \mathop{\arg\min}\limits_{p \geq 0, p^T 1_n = 1}  \left\|p + (\eta w^{(t+1)} + \beta c) / \mu \right\|_2^2, \label{eq:min_p}
    \end{alignat}
\end{subequations}
which are the simplex projection problems that can be solved exactly with $\mathcal{O}(n \log n)$ complexity \cite{Wang2013ProjectionOT}.

By alternately solving \eqref{eq:min_w} and \eqref{eq:min_p}, the $(w^{t}, p^{t})$ converges to the global optimum of \eqref{eq:model_wp} $(w^{\ast}, p^{\ast})$, typically within $t_{\max} = 20$ iterations.
In our experiments, the condition $\sum_{k=1}^{n} w_k^{\ast} \cdot p_k^{\ast} = 0$ almost always holds when $\eta = 0.99 \mu$ and $\mu$ is not too large, the reason for which needs further study.

Finally, the optimization algorithm of the proposed model is summarized in Alg. \ref{alg:model}.
Instead of initializing variables randomly, we set $k_0 = \lfloor \log_2(n) \rfloor + 1$ suggested by \cite{Luxburg2004ATO}.
Then, we initialize the first $k_0$ elements of $w$ as $1 / k_0$, and the last $n - k_0$ elements of $p$ as $1 / (n - k_0)$, while the rest elements of $w$ and $p$ as zeros.
$V$ is initialized by using PHALS($S(w), V_0$) (Alg. 2, \cite{Hou2023APrgressiveHA}), where $V_0$ is an $n \times r$ matrix whose elements are uniformly sampled in the range of $[0, 1]$.
The stopping condition is met if any one of the following three criteria is satisfied:
\begin{enumerate}
    \item The iteration count of while-loop $b$ reaches $1000$.
    \item The relative change of loss $\frac{|\mathrm{loss}(b-1) - \mathrm{loss}(b)|}{|\mathrm{loss(b-1)}|} \leq 10^{-4}$.
    \item The relative change of variables $\Delta_b \leq 10^{-4}$, where
\end{enumerate}
\begin{equation}
    \Delta_b = \tfrac{\|V^{(b)} - V^{(b-1)}\|_F}{\|V^{(b-1)} \|_F} + \tfrac{\|w^{(b)} - w^{(b-1)}\|_2}{\|w^{(b-1)} \|_2}  + \tfrac{\|p^{(b)} - p^{(b-1)}\|_2}{\|p^{(b-1)} \|_2}.
\end{equation}

\subsection{Convergence Analysis}
As mentioned previously, an advantage of the proposed orthogonality regularization is the theoretical convergence guarantee, which is described by the following two theorems.

\begin{theorem} \label{theorem:convergence_loss}
    In Alg. \ref{alg:model}, the objective function of \eqref{eq:model} converges to a finite value.
\end{theorem}
\begin{proof}
    In the proposed optimization algorithm, \eqref{eq:model_V} and \eqref{eq:model_wp} are alternatively solved.
    For the former, one-step PHALS algorithm is applied, and the objective function of \eqref{eq:model_V} is monotonically decreasing according to Lemma \ref{lemma2}.
    For the latter, the global optimum is obtained.
    Therefore, the objective function of \eqref{eq:model} is monotonically decreasing.
    Moreover, the objective function of \eqref{eq:model} has a lower bound:
    \begin{equation}
    \begin{aligned}
        &\tfrac{1}{2} \|S(w) - VV^T \|_F^2 + \beta \langle D(p), VV^T \rangle - \alpha \mathcal{R}(V) \\ 
            &\qquad + \tfrac{\mu - 1}{2} \| S(w) \|_F^2 + \tfrac{\mu}{2} \|D(p) \|_F^2 \\
        &\geq \tfrac{1}{2} \|S(w) - VV^T \|_F^2 - \alpha \mathcal{R}(V) - \tfrac{1}{2}\| S(w) \|_F^2 \\
        &\mathop{\geq}^{\eqref{eq:sup_RV}} \tfrac{1}{2} \|S(w) - VV^T \|_F^2 - \alpha \|V \|_F^2 - \tfrac{1}{2}\| S(w) \|_F^2 \\
        &= \tfrac{1}{2} \|\alpha I_n + S(w) - VV^T \|_F^2 - \alpha \mathrm{Tr}(S(w)) \\
            &\qquad - \tfrac{\alpha^2}{2} \|I_n \|_F^2 - \tfrac{1}{2}\| S(w) \|_F^2 \\
        &\geq - \alpha \mathrm{Tr}(S(w)) - \tfrac{\alpha^2}{2} \|I_n \|_F^2 - \tfrac{1}{2}\| S(w) \|_F^2 \\
        &= - \alpha w_1 \mathrm{Tr}(A^{(1)}) - \tfrac{\alpha^2 n}{2} - \tfrac{1}{2}\| w \|_2^2 \\
        &\geq -\sqrt{n} - \tfrac{\alpha^2 n}{2} - \tfrac{n}{2}.
    \end{aligned}
    \end{equation}
    As a result, the objective function of \eqref{eq:model} converges to a finite value.
\end{proof}

\begin{theorem} \label{theorem:KKT}
    In Alg. \ref{alg:model}, the variable set $(V, w, p)$ converges to a stationary point that satisfies the KKT conditions.
\end{theorem}
\begin{proof}
    Once $V$ converges to a stationary point (i.e., $V^{(b+1)} = V^{(b)}$), $(w, p)$ is natural a stationary point that meets the KKT conditions, since $(w, p)$ is the global optimum of \eqref{eq:model_wp}.
    Therefore, the key is to prove that each $v_j$ converges to a stationary point that meets the KKT conditions.
    Denoting $M = S(w) - \beta D(p) + \alpha (I_n - V_{-j} V_{-j}^{\dag}) - V_{-j} V_{-j}^T$, the KKT conditions w.r.t. $v_j$ are:
    \begin{equation}
        \begin{cases}
            2 v_j v_j^T v_j - 2 M v_j - \lambda = 0 \\
            v_j \geq 0, \quad \lambda \geq 0, \quad v_j \odot \lambda = 0
        \end{cases},
    \end{equation}
    where $\lambda$ is the Lagrange multiplier w.r.t. the constraint $v_j \geq 0$.
    It is easy to verify that the KKT conditions are equivalent to
    \begin{equation}
        \|v_j \|_2^2 \left( v_j - \max(M v_j / \|v_j \|_2^2, 0) \right) = -\|v_j \|_2^2 d = 0,
        \label{eq:KKT_condition}
    \end{equation}
    where $d = \max(M v_j / \|v_j \|_2^2, 0) - v_j$ is an auxiliary variable in the PHALS algorithm.
    When $v_j = 0_n$, $d = 0_n$.

    According to the Theorem 5 in \cite{Hou2023APrgressiveHA}, after an one-step updateing of $v_j$ (i.e., line 10 in Alg. \ref{alg:model}), the objective function of \eqref{eq:model_vj} $f(v_j)$ satisfies:
    \begin{equation}
        f(v_{j}^{b}) \!-\! f(v_{j}^{b+1}) \geq \frac{2 \|v_j^b \|^4 \|d \|^2}{\|d \|^2 \!+\! 3 \|d \| \!\cdot\! \|v_j^b \| \!+\! 3 \|v_j^b \|^2 \!+\! \|M \|_F},
    \end{equation}
    where $b$ is the iteration count.
    Moreover, Theorem \ref{theorem:convergence_loss} shows that $\lim\limits_{b \rightarrow +\infty} f(v_j^{b}) - f(v_j^{b+1}) = 0$, which implies that:
    \begin{equation}
        \lim_{b \rightarrow +\infty} \|v_j^b \|^4 \|d \|^2 = 0,
    \end{equation}
    i.e., the KKT condition \eqref{eq:KKT_condition}.
\end{proof}

\begin{algorithm}[t] \caption{Algorithm of the proposed model} \label{alg:model}
\KwIn{Slices matrice $A^{(k)} \in \R^{n \times n}, k = 1, \dots, n$,\newline
    number of classes $r$, \newline
    hyper-parameters $\alpha, \beta, \mu$.
}
\KwOut{$V, w, p$}
Initialize $k_0 = \lfloor \log_2(n) \rfloor + 1$, $\eta = 0.99 \mu$, $w = p = \text{zeros}(n, 1)$, $V_0 = \text{rand}(n, r)$. \\
$w_{1:k_0} = 1 / k_0, \quad p_{k_0+1:n} = 1 / (n - k_0)$. \\
Calculate $S$ and $D$ by \eqref{eq:S_w} and \eqref{eq:D_p}, respectively. \\
Initialize $V$ by using PHALS($S$, $V_0$) (Alg. 2, \cite{Hou2023APrgressiveHA}). \\
\While{Not convergence}{
    \For{$j = 1, \dots, r$}{
        $V_{-j} = [v_1, \dots, v_{j-1}, v_{j+1}, \dots, v_{r}]$. \\
        Calculate $V_{-j}^{\dag}$ via the reduced SVD \cite{Vasudevan2017AHS}. \\
        $M = S - \beta D + \alpha (I_n - V_{-j} V_{-j}^{\dag}) - V_{-j} V_{-j}^T$ \\
        Updating $v_j$ by using rank-one PHALS($M$, $v_j$).
    }
    \For{$t = 1, \dots, 20$}{
        Updating $w$ by solving \eqref{eq:min_w} (Alg. 1, \cite{Wang2013ProjectionOT}). \\
        Updating $p$ by solving \eqref{eq:min_p} (Alg. 1, \cite{Wang2013ProjectionOT}).
    }
    Calculate $S$ and $D$ by \eqref{eq:S_w} and \eqref{eq:D_p}, respectively.
}
\end{algorithm}

\subsection{Comnputational Complexity Analysis}

The main complexity of Alg. \ref{alg:model} lies in the while-loop, which is analyzed line by line as follows:

In line 8, $V_{-j}^{\dag}$ is calculated via the reduced SVD \cite{Vasudevan2017AHS}, which requires $\mathcal{O}(n (r - 1)^2)$ complexity.

From line 9 to line 10, the PHALS algorithm does not use $M$ directly, but only needs to calculate $M v_j$, i.e.,
\begin{equation}
    M v_j = (S - \beta D) v_j + \alpha v_j - V_{-j} \left[\left(\alpha V_{-j}^{\dag} + V_{-j}^T \right) v_j \right],
\end{equation}
which requires only $\mathcal{O}(n^2 + n (r-1))$ complexity.
Therefore, the complexity from line 6 to line 10 is $\mathcal{O}(n^2 r + n r^2)$.

According to \cite{Wang2013ProjectionOT}, the complexities of line 12 and line 13 are $\mathcal{O}(n \log n)$.

In line 14, \eqref{eq:S_w} and \eqref{eq:D_p} involve the summation of $n$ $\R^{n \times n}$ matrices.
However, each $A^{(k)}$ is a sparse matrix with only $n$ non-zero elements. 
Due to this sparsity, \eqref{eq:S_w} and \eqref{eq:D_p} require only $\mathcal{O}(n^2)$ computational complexity.

In summary, each iteration's complexity of the proposed algorithm is $\mathcal{O}(n^2 r + n r^2 + n \log n)$.

\section{Experiment} \label{sec:5}
\subsection{Experiment Settings}
We compared the proposed method with five fixed similarity based methods: SymNMF \cite{Kuang2015SymNMF, Kuang2012SymmetricNM}, ANLS \cite{Zhu2018DroppingSF}, sBSUM \cite{Shi2017InexactBC}, PHALS \cite{Hou2023APrgressiveHA}, GSNMF \cite{Gao2018GraphRS}; and four state-of-the-art adaptive similarity based methods: NMFAN \cite{Huang2017NonnegativeMF}, CGSymNMF \cite{Jia2020ClusteringAG}, RBSMF \cite{Wang2022RobustBS}, S$^{3}$NMF \cite{Jia2022S3NMF}.
For all methods, we adopted the same initial $V$ and $k$-NN graph $S_0$, where elements in $V$ were uniformly sampled in the range of $[0, 1]$, and $S_0$ was constructed according to \eqref{eq:S_knn} with $k = \lfloor \log_2(n) \rfloor + 1$ as suggested by \cite{Luxburg2004ATO} and the kernel function $\kappa(x_i, x_j)$ defined by the self-tuning method \cite{ZelnikManor2004SelfTuningSC}.
To generate the clustering result of the proposed model, we performed the spectral clustering method \cite{AY2002OnSC} on the augmented similarity matrix $Z$ (defined in \eqref{eq:aug_affinity}) 20 times, and reported the average performance.

\begin{table}[!t]
    \centering
    \caption{Details of Datasets}
    \begin{tabular}{cccc}
    \toprule
    Dataset     & Dimension         & \# Sample ($n$) & \# Cluster ($r$) \\
    \midrule
    SEEDS       & $7$               & $210$     & $3$   \\
    YaleB       & $32 \times 32$    & $165$     & $15$ \\
    ORL         & $32 \times 32$    & $400$     & $40$ \\
    CHART       & $60$              & $600$     & $6$ \\
    USPS-1000   & $16 \times 16$    & $1000$    & $10$ \\
    COIL20      & $32 \times 32$    & $1440$    & $20$ \\
     NIST-2000  & $28 \times 28$    & $2000$    & $10$ \\
    Semeion     & $16 \times 16$    & $1593$    & $10$ \\
    \bottomrule
    \end{tabular}
    \label{tab:datasets}
\end{table}

For fair comparison, the hyper-parameters of ANLS, NMFAN, CGSymNMF, RBSMF and S$^{3}$NMF were exhaustively searched in the grid provided in the original papers.
For GNMF and GSNMF, hyper-parameters were tuned in the range of $\mathrm{logspace}(-2, 4, 50)$, i.e., $50$ numbers spaced evenly on a log scale from $10^{-2}$ to $10^{4}$.
For SymNMF, PHALS and sBSUM, there are no hyper-parameters to tune.
For the proposed model, $\alpha$, $\beta$ and $\mu$ were tuned in the range of $\{0.01, 0.03, 0.07, 0.1, 0.3, 0.7, 1\}, \{1, 5, 10, 50, 100, 500, 1000 \}$ and $\{0.05, 0.07, 0.1, 0.3, 0.5, 0.7, 1 \}$, respectively.
Then, we reported their average performance and the associated standard deviation (std) of 20 repetitions.
We evaluated all the methods on eight commonly used datasets. 
The detailed information about these datasets is summarized in Table \ref{tab:datasets}.

We adopted two commonly used metrics, namely clustering accuracy (ACC) and normalized mutual information (NMI), to evaluate all the methods.
Both values of ACC and NMI lie in the range of $[0, 1]$, and the larger, the better.

\subsection{Comparisons of Clustering Performance}
\begin{table*}[!t]
    \centering
    \tabcolsep=1.2mm
    \caption{Clustering Performance on Each Dataset.}
    \label{tab:clustering}
    \begin{threeparttable}
    \begin{tabular}{c|cccccccc}
    \toprule
    ACC & SEEDS & YaleB & ORL & CHART & USPS-1000 & COIL20 & MNIST-2000 & Semeion \\
    \midrule
    SymNMF                      & $0.658 \pm 0.126$ & $0.458 \pm 0.018$ & $0.616 \pm 0.022$ & $0.632 \pm 0.059$ & $0.514 \pm 0.041$ & $0.503 \pm 0.038$ & $0.535 \pm 0.042$ & $0.547 \pm 0.066$ \\
    ANLS                        & $0.753 \pm 0.132$ & $0.459 \pm 0.017$ & $0.614 \pm 0.013$ & $0.643 \pm 0.056$ & $0.531 \pm 0.044$ & $0.557 \pm 0.037$ & $0.567 \pm 0.053$ & $0.576 \pm 0.052$ \\
    sBSUM                       & $0.840 \pm 0.029$ & $0.461 \pm 0.013$ & $0.620 \pm 0.014$ & $0.595 \pm 0.057$ & $0.562 \pm 0.047$ & $0.644 \pm 0.066$ & $0.565 \pm 0.067$ & $0.615 \pm 0.060$ \\
    PHALS                       & $0.789 \pm 0.120$ & \underline{$0.467 \pm 0.015$} & $0.621 \pm 0.015$ & $0.646 \pm 0.062$ & $0.532 \pm 0.041$ & $0.658 \pm 0.043$ & $0.603 \pm 0.051$ & $0.594 \pm 0.043$ \\
    GSNMF                       & $0.647 \pm 0.092$ & $0.455 \pm 0.024$ & $0.620 \pm 0.017$ & $0.647 \pm 0.087$ & $0.498 \pm 0.050$ & $0.504 \pm 0.034$ & $0.493 \pm 0.043$ & $0.495 \pm 0.038$ \\
    NMFAN$^{\ddagger}$     & $0.876 \pm 0.008$ & $0.450 \pm 0.017$ & \underline{$0.678 \pm 0.007$} & $0.627 \pm 0.075$ & $0.463 \pm 0.029$ & $0.660 \pm 0.016$ & $0.516 \pm 0.044$ & $0.611 \pm 0.044$ \\
    CGSymNMF$^{\ddagger}$  & $0.755 \pm 0.109$ & $0.459 \pm 0.032$ & $0.621 \pm 0.025$ & $0.661 \pm 0.080$ & $0.487 \pm 0.042$ & $0.591 \pm 0.048$ & $0.506 \pm 0.042$ & $0.480 \pm 0.037$ \\
    RBSMF$^{\ddagger}$     & \pmb{$0.897 \pm 0.012$} & $0.462 \pm 0.014$ & $0.645 \pm 0.007$ & $0.568 \pm 0.000$ & $0.571 \pm 0.002$ & \underline{$0.800 \pm 0.000$} & \underline{$0.665 \pm 0.011$} & $0.637 \pm 0.041$ \\
    S$^{3}$NMF$^{\ddagger}$          & $0.835 \pm 0.041$ & $0.450 \pm 0.008$ & $0.617 \pm 0.007$ & \pmb{$0.829 \pm 0.016$} & \pmb{$0.650 \pm 0.021$} & $0.781 \pm 0.015$ & $0.664 \pm 0.020$ & \underline{$0.704 \pm 0.017$} \\
    proposed$^{\ddagger}$       & \underline{$0.887 \pm 0.016$} & \pmb{$0.539 \pm 0.020$} & \pmb{$0.683 \pm 0.014$} & \underline{$0.733 \pm 0.057$} & \underline{$0.610 \pm 0.027$} & \pmb{$0.833 \pm 0.015$} & \pmb{$0.698 \pm 0.023$} & \pmb{$0.734 \pm 0.055$} \\
    \bottomrule
    \toprule
    NMI & SEEDS & YaleB & ORL & CHART & USPS-1000 & COIL20 & MNIST-2000 & Semeion \\
    \midrule
    SymNMF                      & $0.400 \pm 0.147$ & $0.522 \pm 0.017$ & $0.789 \pm 0.010$ & $0.710 \pm 0.066$ & $0.529 \pm 0.030$ & $0.689 \pm 0.021$ & $0.523 \pm 0.026$ & $0.553 \pm 0.043$ \\
    ANLS                        & $0.551 \pm 0.143$ & $0.526 \pm 0.012$ & $0.787 \pm 0.007$ & $0.767 \pm 0.041$ & $0.553 \pm 0.028$ & $0.743 \pm 0.027$ & $0.562 \pm 0.042$ & $0.585 \pm 0.036$ \\
    sBSUM                       & $0.646 \pm 0.029$ & $0.527 \pm 0.009$ & $0.791 \pm 0.007$ & $0.784 \pm 0.040$ & $0.578 \pm 0.030$ & $0.833 \pm 0.026$ & $0.620 \pm 0.041$ & $0.632 \pm 0.026$ \\
    PHALS                       & $0.580 \pm 0.141$ & $0.530 \pm 0.008$ & $0.789 \pm 0.007$ & $0.790 \pm 0.028$ & $0.568 \pm 0.023$ & $0.819 \pm 0.024$ & $0.613 \pm 0.036$ & $0.617 \pm 0.022$ \\
    GSNMF                       & $0.392 \pm 0.091$ & $0.522 \pm 0.017$ & $0.786 \pm 0.009$ & $0.714 \pm 0.053$ & $0.508 \pm 0.042$ & $0.678 \pm 0.025$ & $0.423 \pm 0.041$ & $0.488 \pm 0.030$ \\
    NMFAN$^{\ddagger}$     & $0.642 \pm 0.014$ & $0.501 \pm 0.014$ & \underline{$0.830 \pm 0.006$} & $0.626 \pm 0.025$ & $0.445 \pm 0.037$ & $0.756 \pm 0.017$ & $0.457 \pm 0.027$ & $0.554 \pm 0.031$ \\
    CGSymNMF$^{\ddagger}$  & $0.548 \pm 0.096$ & $0.522 \pm 0.021$ & $0.800 \pm 0.013$ & $0.704 \pm 0.065$ & $0.495 \pm 0.041$ & $0.769 \pm 0.031$ & $0.501 \pm 0.036$ & $0.465 \pm 0.035$ \\
    RBSMF$^{\ddagger}$     & \underline{$0.687 \pm 0.021$} & \underline{$0.534 \pm 0.003$} & $0.806 \pm 0.004$ & $0.801 \pm 0.001$ & \pmb{$0.627 \pm 0.002$} & \underline{$0.892 \pm 0.000$} & \underline{$0.654 \pm 0.008$} & $0.649 \pm 0.016$ \\
    S$^{3}$NMF$^{\ddagger}$          & $0.647 \pm 0.055$ & $0.512 \pm 0.006$ & $0.787 \pm 0.004$ & \pmb{$0.809 \pm 0.017$} & $0.625 \pm 0.016$ & $0.860 \pm 0.007$ & $0.636 \pm 0.015$ & \underline{$0.665 \pm 0.011$} \\
    proposed$^{\ddagger}$       & \pmb{$0.709 \pm 0.021$} & \pmb{$0.586 \pm 0.012$} & \pmb{$0.846 \pm 0.006$} & \pmb{$0.809 \pm 0.011$} & \pmb{$0.627 \pm 0.015$} & \pmb{$0.923 \pm 0.012$} & \pmb{$0.691 \pm 0.011$} & \pmb{$0.678 \pm 0.021$} \\
    \bottomrule
    \end{tabular}
    \begin{tablenotes}
        \footnotesize
        \item[$\ddagger:$] Adaptive similarity methods.
    \end{tablenotes}
    \end{threeparttable}
\end{table*}

Table \ref{tab:clustering} shows the clustering performance of all the methods on each dataset, where the best performance under each metric is marked by bold, and the second best is underlined.
From Table \ref{tab:clustering}, we can observe that:
\begin{itemize}
    \item Our method significantly outperforms the compared methods, achieving the highest ACC/NMI values in most cases ($13/16$) and the second highest in the rest ($3/16$).
    \item The adaptive similarity methods outperforms the fixed similarity methods. There is only one case where a fixed similarity method achieves the second-best performance: the ACC of PHALS on the YaleB dataset.
    \item Among the compared methods, RBSMF and S$^{3}$NMF have good performance. The reasons may be as follows: 
    In RBSMF, the similarity matrix is assumed to have a bi-stochastic structure, making it easier to learn a high-quality similarity matrix. Moreover, it adopts a robust loss function instead of the Forbenius norm in NMF to handle outliers. 
    In S$^{3}$NMF, the similarity matrix is iteratively boosted according to multiple clustering results, making it more robust to handle poor initial similarity matrix.
\end{itemize}

\subsection{Learned Coefficients of $w$ and $p$}
\begin{figure*}[!t]
    \centering
    \includegraphics[width=\linewidth]{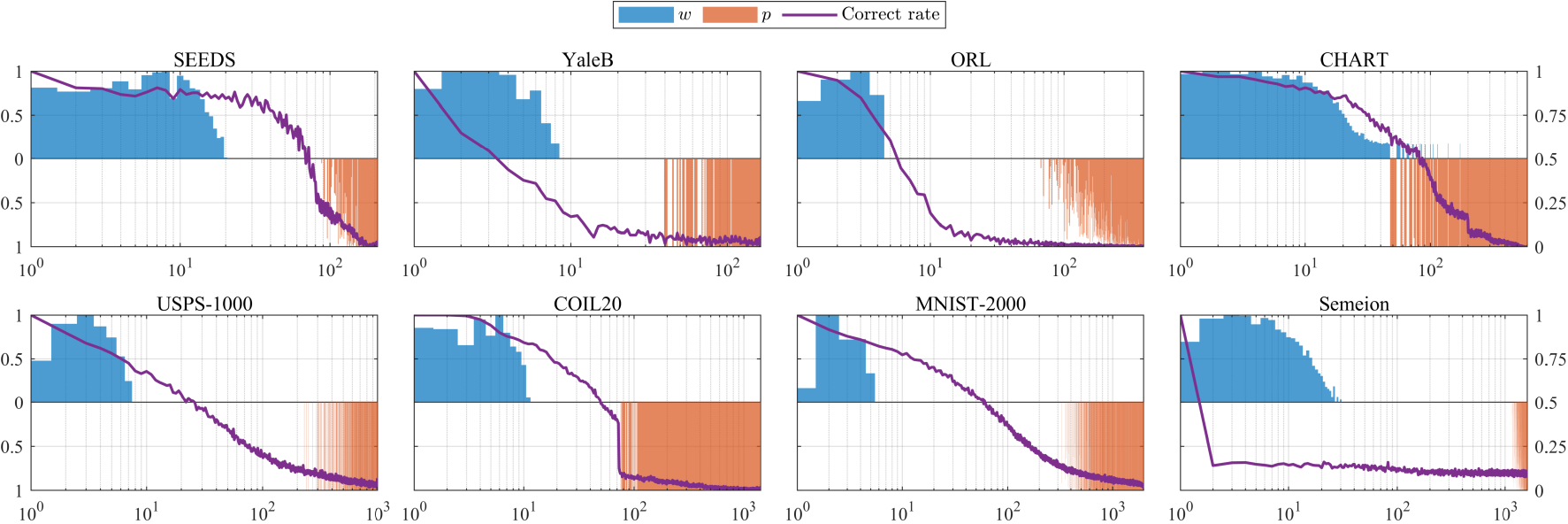}
    \caption{Correct rate and the learned $w$ and $p$ of each dataset. The $x$-axis represents the $k$-th nearest neighbors on a logarithmic scale. The upper and lower parts of the left $y$-axis represents the coordinates of $w$ and $p$ respectively, and the right y-axis is the coordinate of correct rate. For better view, $w$ and $p$ are normalized to the range $[0, 1]$.}
    \label{fig:wpACC}
\end{figure*}

The correct rate and the learned $w$ and $p$ of each dataset are shown in Fig. \ref{fig:wpACC}.
Generally, as $k$ increases, correct rate decreases.
This phenomenon causes $w$ and $p$ automatically learn $k$ nearest neighbors and $t$ farthest neighbors, for some $k$ and $t$, respectively, with the weights being consistent with the correct rate.
That is, the larger the correct rate, the larger the $w$ and the smaller the $p$. 
Moreover, $w$ and $p$ have no overlap, indicating that the condition in \eqref{eq:wp_condition} is satisfied, i.e., the non-convex problem \eqref{eq:model_wp1} can be solved equivalently by the strongly convex problem \eqref{eq:model_wp}.

\subsection{Hyper-parameters Analysis}
There are three hyper-parameters $\alpha, \beta$ and $\mu$ in the proposed model, which control the contributions of the orthogonality regularization, the dissimilarity regularization, and the density regularization, respectively.
In this subsection, we investigate their influence on the ACC of the proposed model in Fig. \ref{fig:hp}, we can be observed that:
\begin{itemize}
    \item A larger $\beta$ usually leads to higher ACCs, but at the same time, lower $\alpha$ often results in inferior ACCs. 
    This is because when $\beta \langle D(p), VV^T \rangle$ is dominant, it becomes easy to learn a $V$ of all zeros. 
    However, when $\alpha$ is large, $\|V \|_F^2$ is maximized, which helps avoid this undesirable solution.
    \item Taking into account the ACC of all datasets, we suggest $\mu = \alpha = 0.1, \beta = 10$ as default setting.
\end{itemize}

\begin{figure*}[!t]
    \centering
    \subfloat{\includegraphics[width=0.245\linewidth]{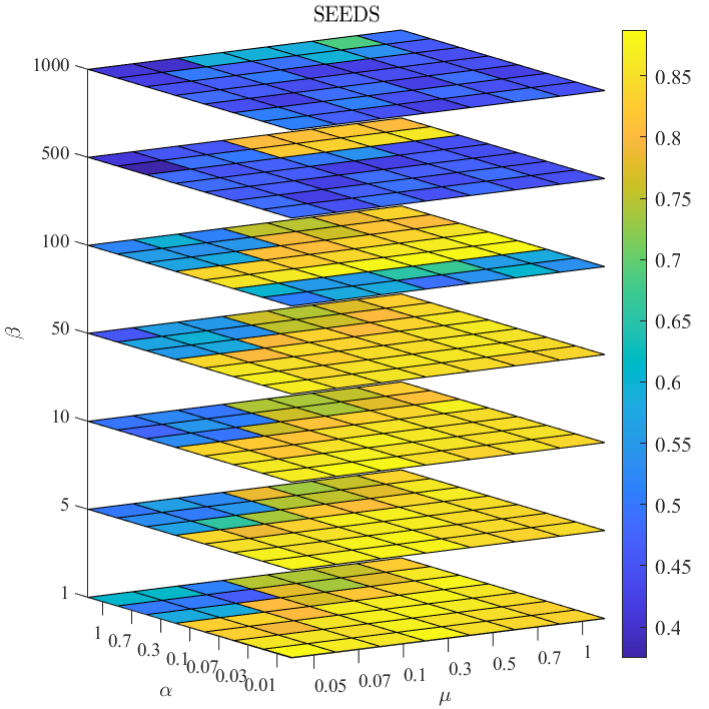}} \hfill
    \subfloat{\includegraphics[width=0.245\linewidth]{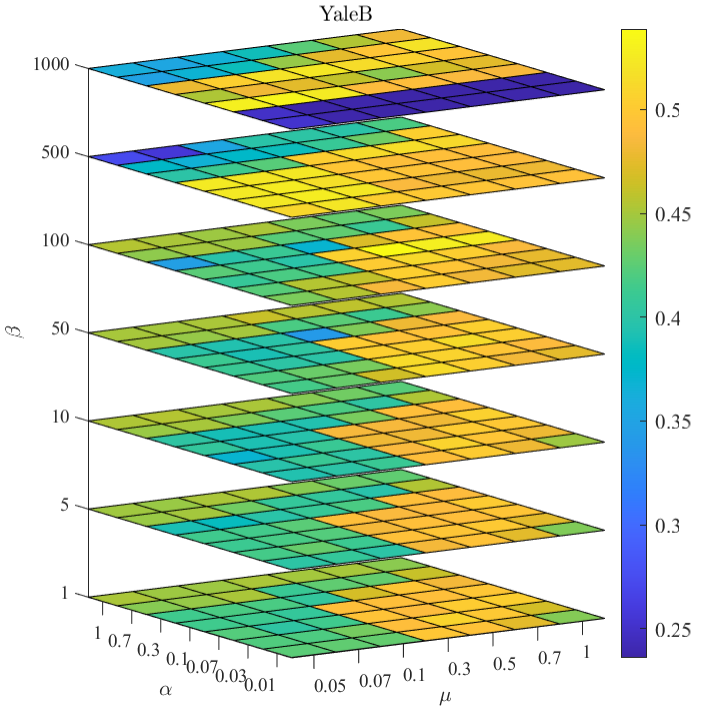}} \hfill
    \subfloat{\includegraphics[width=0.245\linewidth]{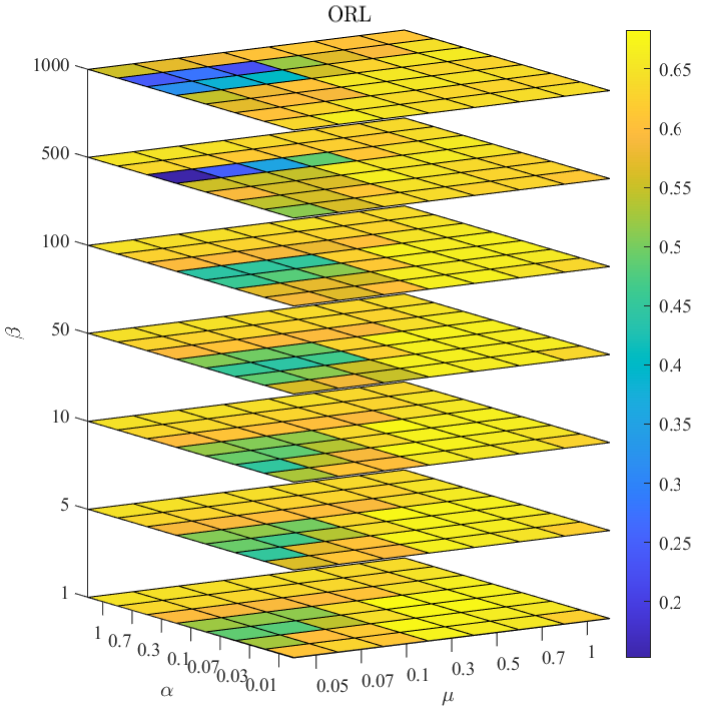}}   \hfill
    \subfloat{\includegraphics[width=0.245\linewidth]{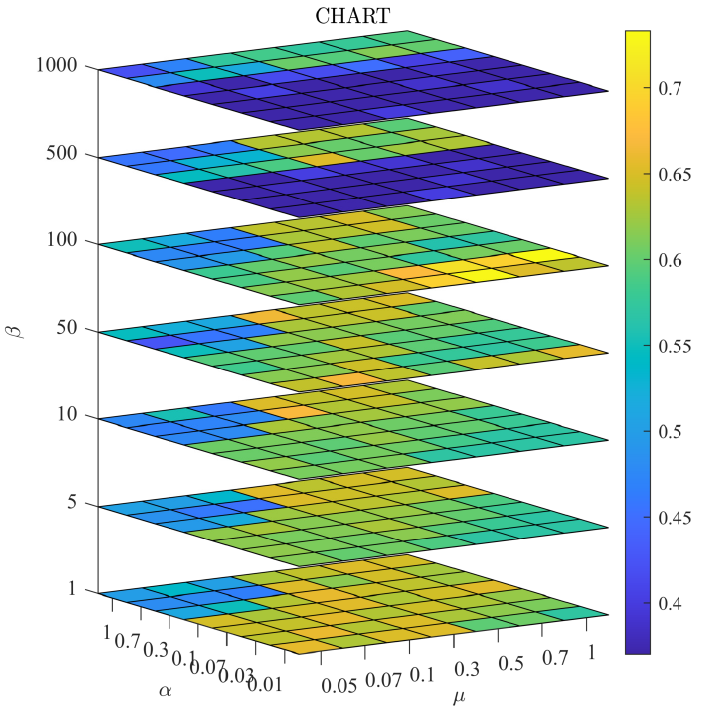}} \hfill \\
    \subfloat{\includegraphics[width=0.245\linewidth]{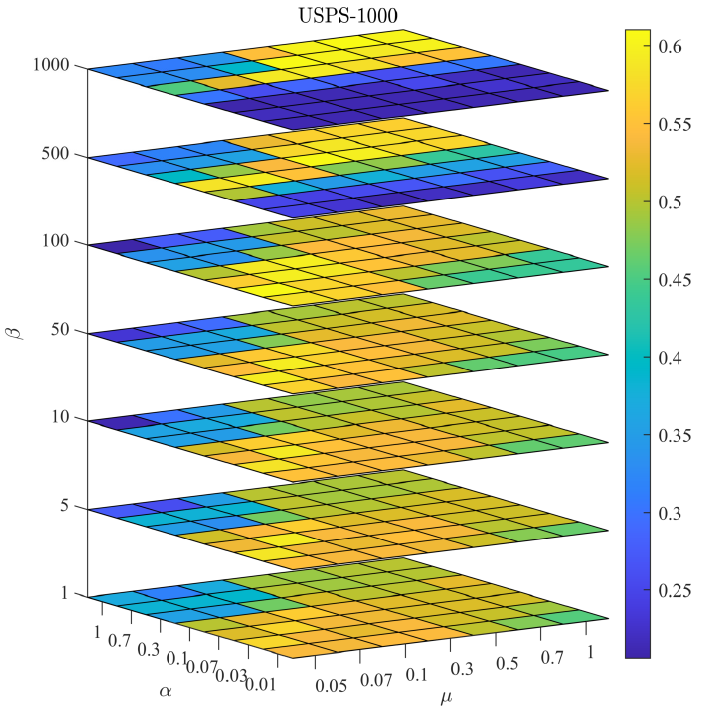}}  \hfill
    \subfloat{\includegraphics[width=0.245\linewidth]{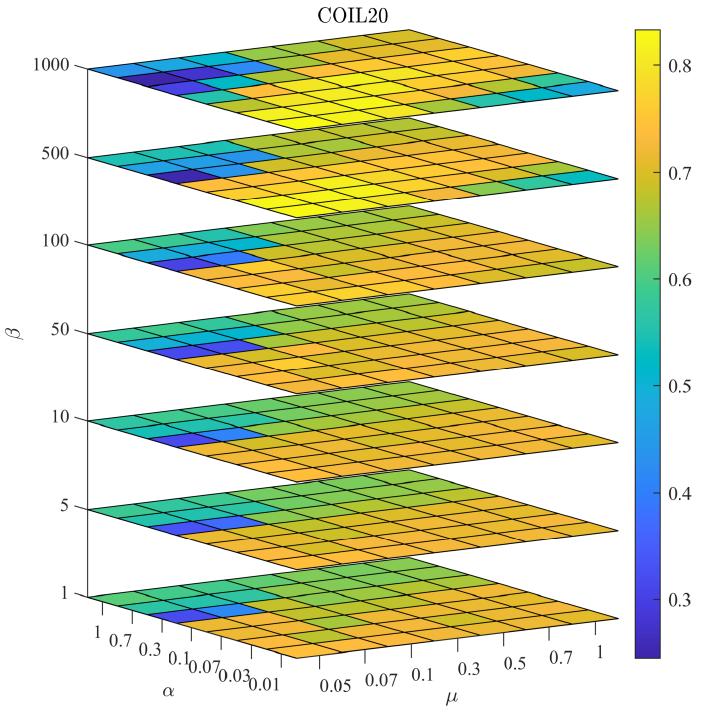}}\hfill
    \subfloat{\includegraphics[width=0.245\linewidth]{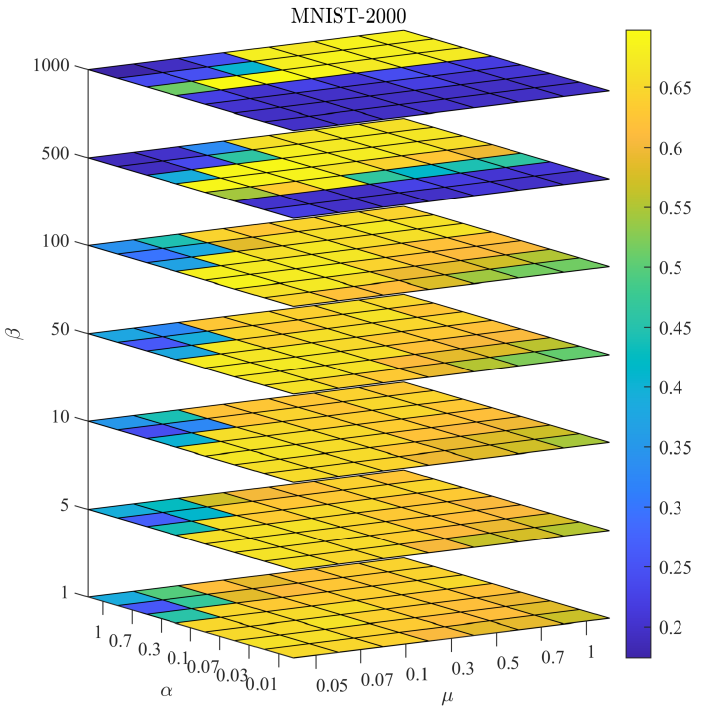}} \hfill
    \subfloat{\includegraphics[width=0.245\linewidth]{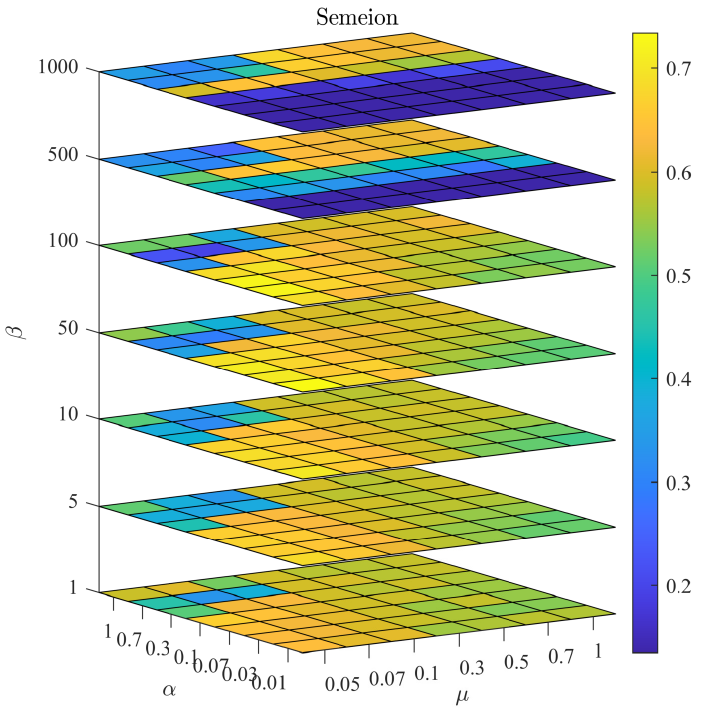}} \hfill \\
    \caption{Average values of ACC of the proposed model with different values of $\alpha$, $\beta$ and $\mu$. They are all 4-D figures, where the fourth direction is indicated by the color with the corresponding color bar.}
    \label{fig:hp}
\end{figure*}

\subsection{Ablation Study} \label{sec:ablation}
In this subsection, we evaluated the importance of different components of the proposed model.
Specifically, we remove the orthogonality regularization and the dissimilarity matrix $D(p)$ from the proposed model \eqref{eq:model}, as shown in \eqref{eq:model_woOrth} and \eqref{eq:model_woD}, respectively, and remove them both as shown in \eqref{eq:model_woBoth}.
\begin{equation}
\begin{aligned}
    &\min_{V, w, p \geq 0} \tfrac{1}{2} \|S(w) - VV^T \|_F^2 + \beta \langle D(p), VV^T \rangle \\ 
    &\qquad \qquad  + \tfrac{\mu - 1}{2} \| S(w) \|_F^2 + \tfrac{\mu}{2} \|D(p) \|_F^2 \\
    &\text{s.t.} \quad w^T 1_n = p^T 1_n = 1, w^T p = 0,
    \label{eq:model_woOrth}
\end{aligned}
\end{equation}
\begin{equation}
\begin{aligned}
    &\min_{V, w \geq 0} \tfrac{1}{2} \|S(w) - VV^T \|_F^2  - \alpha \mathcal{R}(V) + \tfrac{\mu}{2} \| S(w) \|_F^2  \\
    &\text{s.t.} \quad w^T 1_n = 1,
    \label{eq:model_woD}
\end{aligned}
\end{equation}
\begin{equation}
    \min_{V, w \geq 0} \tfrac{1}{2} \|S(w) \!-\! VV^T \|_F^2 + \tfrac{\mu}{2} \| S(w) \|_F^2 \; \text{s.t.} \; w^T 1_n = 1.
    \label{eq:model_woBoth}
\end{equation}
For a fair comparison, the hyperparamaters of \eqref{eq:model_woOrth}, \eqref{eq:model_woD} and \eqref{eq:model_woBoth} were carefully tuned on a larger and more dense grid.
Specifically, for \eqref{eq:model_woOrth}, $\mu$ was tuned in the range of $\{0.03, 0.05, 0.07, 0.1, 0.3, 0.5, 0.7, 1, 1.2 \}$, and $\beta$ was tuned in the range of $\{1, 5, 10, 30, 70, 100, 300, 700, 1000 \}$; in \eqref{eq:model_woD}, $\mu$ and $\alpha$ were tuned in the range of $\{0.1, 0.3, 0.5, 0.7, 1, 3, 5, 7, 10 \}$; in \eqref{eq:model_woBoth}, $\mu$ was tuned in the range of $\mathrm{logspace}(-2, 2, 50)$.

\begin{table*}[!t]
    \centering
    \tabcolsep=1.2mm
    \caption{Ablation study of the proposed model.}
    \label{tab:ablation}
    \begin{tabular}{ccc|cccccccc}
    \toprule
    & $D$ & Orth & SEEDS & YaleB & ORL & CHART & USPS-1000 & COIL20 & MNIST-2000 & Semeion \\
    \midrule
    \eqref{eq:model_woBoth} &            &            & $0.790 \pm 0.071$ & $0.513 \pm 0.007$ & $0.668 \pm 0.010$ & \pmb{$0.864 \pm 0.023$} & $0.412 \pm 0.018$ & $0.560 \pm 0.037$ & $0.487 \pm 0.003$ & $0.536 \pm 0.019$ \\
    \eqref{eq:model_woOrth} & \checkmark &            & \pmb{$0.896 \pm 0.021$} & $0.492 \pm 0.016$ & $0.635 \pm 0.005$ & $0.672 \pm 0.093$ & $0.502 \pm 0.040$ & $0.749 \pm 0.037$ & $0.657 \pm 0.044$ & $0.601 \pm 0.035$ \\
    \eqref{eq:model_woD}    &            & \checkmark & $0.878 \pm 0.030$ & \underline{$0.518 \pm 0.022$} & \underline{$0.672 \pm 0.010$} & $0.686 \pm 0.115$ & \underline{$0.608 \pm 0.020$} & \underline{$0.784 \pm 0.051$} & \underline{$0.661 \pm 0.038$} & \pmb{$0.760 \pm 0.033$} \\
    \eqref{eq:model}        & \checkmark & \checkmark & \underline{$0.887 \pm 0.016$} & \pmb{$0.539 \pm 0.020$} & \pmb{$0.683 \pm 0.014$} & \underline{$0.733 \pm 0.057$} & \pmb{$0.610 \pm 0.027$} & \pmb{$0.833 \pm 0.015$} & \pmb{$0.698 \pm 0.023$} & \underline{$0.734 \pm 0.055$} \\
    \bottomrule
    \toprule
    & $D$ & Orth & SEEDS & YaleB & ORL & CHART & USPS-1000 & COIL20 & MNIST-2000 & Semeion \\
    \midrule
    \eqref{eq:model_woBoth} &            &            & $0.545 \pm 0.090$ & $0.554 \pm 0.061$ & $0.798 \pm 0.002$ & \pmb{$0.824 \pm 0.019$} & $0.386 \pm 0.013$ & $0.727 \pm 0.030$ & $0.441 \pm 0.007$ & $0.493 \pm 0.019$ \\
    \eqref{eq:model_woOrth} & \checkmark &            & \pmb{$0.722 \pm 0.028$} & $0.544 \pm 0.015$ & $0.830 \pm 0.005$ & $0.807 \pm 0.012$ & $0.531 \pm 0.032$ & $0.887 \pm 0.017$ & $0.623 \pm 0.030$ & $0.621 \pm 0.021$ \\
    \eqref{eq:model_woD}    &            & \checkmark & $0.695 \pm 0.035$ & \underline{$0.575 \pm 0.013$} & \underline{$0.832 \pm 0.005$} & $0.800 \pm 0.007$ & \underline{$0.618 \pm 0.022$} & \underline{$0.888 \pm 0.023$} & \underline{$0.658 \pm 0.015$} & \pmb{$0.682 \pm 0.015$} \\
    \eqref{eq:model}        & \checkmark & \checkmark & \underline{$0.709 \pm 0.021$} & \pmb{$0.586 \pm 0.012$} & \pmb{$0.846 \pm 0.006$} & \underline{$0.809 \pm 0.011$} & \pmb{$0.627 \pm 0.015$} & \pmb{$0.923 \pm 0.012$} & \pmb{$0.691 \pm 0.011$} & \underline{$0.678 \pm 0.021$} \\
    \bottomrule
    \end{tabular}
\end{table*}

The clustering performances of these models are shown in Table \ref{tab:ablation}.
Comparing the results of \eqref{eq:model_woOrth} and \eqref{eq:model_woBoth} as well as \eqref{eq:model} and \eqref{eq:model_woD}, it can be seen that the dissimilarity matrix can improve the clustering performance.
Similarly, comparing the results of \eqref{eq:model_woD} and \eqref{eq:model_woBoth} as well as \eqref{eq:model} and \eqref{eq:model_woOrth}, it can be seen that the proposed orthogonality regularization can significantly improve the clustering performance.
These two improvements are especially evident on the COIL20, MNIST-2000 and Semeion datasets.
It may be because the number of samples $n$ in these datasets is relatively large, making it difficult to search for high-quality $S$ in an $n$-dimensional space.
By introducing $D$ and enhancing the orthogonality regularization of $V$, the discriminative ability of $S$ is improved, which makes $S$ easier to learn.

\subsection{Advantage of the proposed orthogonality regularization}
In this subsection, we analyzed the advantage of the proposed orthogonality regularization $\mathcal{R}(V)$ compared with the aforementioned $\mathcal{R}_{\text{off-diag}} = -\sum_{i=1}^{r} \sum_{j \ne i} v_i^T v_j$ \cite{Li2015TwoEA} and $\mathcal{R}_{\text{log-det}} = \log\det(V^T V)$ \cite{Liu2017LargeCN}.
After replacing $\mathcal{R}(V)$ in \eqref{eq:model} with $\mathcal{R}_{\text{off-diag}}$ and $\mathcal{R}_{\text{log-det}}$, multiplicative update (MU) algorithm is applied to solve the models.
As a controlled experiment, we also apply MU to solve \eqref{eq:model} using our proposed regularization.
Moreover, the model \eqref{eq:model_woOrth} without orthogonality regularization is used as a blank control reference.

\begin{table*}[!t]
    \centering
    \tabcolsep=1.2mm
    \caption{Clustering performance of the proposed model with different $\mathcal{R}(V)$.}
    \label{tab:ablation_model_orth}
    \begin{tabular}{c|cccccccc}
    \toprule
    ACC & SEEDS & YaleB & ORL & CHART & USPS-1000 & COIL20 & MNIST-2000 & Semeion \\
    \midrule
    Model \eqref{eq:model_woOrth}       & \underline{$0.878 \pm 0.030$} & \underline{$0.518 \pm 0.022$} & \underline{$0.672 \pm 0.010$} & \underline{$0.686 \pm 0.115$} & \underline{$0.608 \pm 0.020$} & \underline{$0.784 \pm 0.051$} & $0.661 \pm 0.038$ & \pmb{$0.760 \pm 0.033$} \\
    $\mathcal{R}_{\text{off-diag}}(V)$  & $0.853 \pm 0.054$ & $0.500 \pm 0.013$ & $0.666 \pm 0.015$ & $0.662 \pm 0.069$ & $0.517 \pm 0.041$ & $0.715 \pm 0.022$ & $0.656 \pm 0.037$ & $0.591 \pm 0.066$ \\
    $\mathcal{R}_{\text{log-det}}(V)$   & $0.846 \pm 0.072$ & $0.498 \pm 0.011$ & $0.652 \pm 0.016$ & $0.661 \pm 0.070$ & $0.565 \pm 0.038$ & $0.711 \pm 0.050$ & $0.678 \pm 0.031$ & $0.625 \pm 0.068$ \\
    $\mathcal{R}(V)$  by MU             & $0.870 \pm 0.017$ & $0.490 \pm 0.016$ & \underline{$0.672 \pm 0.015$} & $0.661 \pm 0.070$ & $0.569 \pm 0.034$ & $0.718 \pm 0.018$ & \underline{$0.679 \pm 0.029$} & $0.640 \pm 0.065$ \\
    Proposed $\mathcal{R}(V)$           & \pmb{$0.887 \pm 0.016$} & \pmb{$0.539 \pm 0.020$} & \pmb{$0.683 \pm 0.014$} & \pmb{$0.733 \pm 0.057$} & \pmb{$0.610 \pm 0.027$} & \pmb{$0.833 \pm 0.015$} & \pmb{$0.698 \pm 0.023$} & \underline{$0.734 \pm 0.055$} \\
    \bottomrule
    \toprule
    NMI & SEEDS & YaleB & ORL & CHART & USPS-1000 & COIL20 & MNIST-2000 & Semeion \\
    \midrule
    Model \eqref{eq:model_woOrth}       & \pmb{$0.722 \pm 0.028$} & $0.544 \pm 0.015$ & $0.830 \pm 0.005$ & \underline{$0.807 \pm 0.012$} & $0.531 \pm 0.032$ & \underline{$0.887 \pm 0.017$} & $0.623 \pm 0.030$ & $0.621 \pm 0.021$ \\
    $\mathcal{R}_{\text{off-diag}}(V)$  & $0.654 \pm 0.083$ & \underline{$0.555 \pm 0.010$} & $0.830 \pm 0.007$ & $0.799 \pm 0.015$ & $0.542 \pm 0.022$ & $0.831 \pm 0.027$ & $0.645 \pm 0.025$ & $0.604 \pm 0.040$ \\
    $\mathcal{R}_{\text{log-det}}(V)$   & $0.632 \pm 0.106$ & $0.552 \pm 0.009$ & $0.829 \pm 0.008$ & $0.800 \pm 0.021$ & $0.581 \pm 0.027$ & $0.833 \pm 0.027$ & $0.660 \pm 0.024$ & $0.623 \pm 0.035$ \\
    \eqref{eq:model} by MU              & $0.672 \pm 0.026$ & $0.552 \pm 0.011$ & \underline{$0.832 \pm 0.007$} & $0.802 \pm 0.019$ & \underline{$0.589 \pm 0.022$} & $0.833 \pm 0.026$ & \underline{$0.666 \pm 0.025$} & \underline{$0.631 \pm 0.035$} \\
    Proposed $\mathcal{R}(V)$           & \underline{$0.709 \pm 0.021$} & \pmb{$0.586 \pm 0.012$} & \pmb{$0.846 \pm 0.006$} & \pmb{$0.809 \pm 0.011$} & \pmb{$0.627 \pm 0.015$} & \pmb{$0.923 \pm 0.012$} & \pmb{$0.691 \pm 0.011$} & \pmb{$0.678 \pm 0.021$} \\
    \bottomrule
    \end{tabular}
\end{table*}

The clustering performances of these models are shown in Table \ref{tab:ablation_model_orth}.
Surprisingly, the performances of the MU-based methods are worse than model \eqref{eq:model_woOrth} without orthogonality regularization.
This is because MU-based methods are more sensitive to the initializations, as reflected in their larger std.
Nevertheless, benefiting from the effectiveness of the PHALS algorithm, the performance of the proposed $\mathcal{R}(V)$ achieves the best performance.
Moreover, among the MU-based methods, the proposed $\mathcal{R}(V)$ has a slight advantage.

Furthermore, the proposed orthogonality regularization can be applied in both NMF and SymNMF.
As an example, we propose a toy model named approximatively orthogonal SymNMF (AOSymNMF) as follows:
\begin{equation}
    \min_{V \geq 0} \frac{1}{2} \|S - VV^T \|_F^2 - \alpha \mathcal{R}(V).
    \label{eq:AOSymNMF}
\end{equation}
Similarly, we replaced $\mathcal{R}(V)$ in \eqref{eq:AOSymNMF} with $\mathcal{R}_{\text{off-diag}}$ and $\mathcal{R}_{\text{log-det}}$ and applied MU to solve them.
As a controlled experiment, we also applied MU to solve \eqref{eq:AOSymNMF} using our proposed regularization.
Additionally, we use SymNMF and PHALS as blank control references.
The hyper-parameter $\alpha$ in AOSymNMF is exhaustively searched in $\mathrm{logspace}(-2, 4, 50)$.

\begin{table*}[!t]
    \centering
    \tabcolsep=1.2mm
    \caption{Clustering performance of the AOSymNMF \eqref{eq:AOSymNMF} with different $\mathcal{R}(V)$.}
    \label{tab:ablation_AOSymNMF}
    \begin{tabular}{c|cccccccc}
    \toprule
    ACC & SEEDS & YaleB & ORL & CHART & USPS-1000 & COIL20 & MNIST-2000 & Semeion \\
    \midrule
    SymNMF                              & $0.658 \pm 0.126$ & $0.458 \pm 0.018$ & $0.616 \pm 0.022$ & $0.632 \pm 0.059$ & $0.514 \pm 0.041$ & $0.503 \pm 0.038$ & $0.535 \pm 0.042$ & $0.547 \pm 0.066$ \\
    PHALS                               & $0.753 \pm 0.132$ & $0.459 \pm 0.017$ & $0.614 \pm 0.013$ & \pmb{$0.643 \pm 0.056$} & $0.531 \pm 0.044$ & $0.557 \pm 0.037$ & $0.567 \pm 0.053$ & $0.576 \pm 0.052$ \\
    $\mathcal{R}_{\text{off-diag}}(V)$  & $0.672 \pm 0.115$ & \underline{$0.463 \pm 0.022$} & $0.619 \pm 0.011$ &\underline{$0.634 \pm 0.075$} & $0.537 \pm 0.049$ & $0.512 \pm 0.038$ & $0.549 \pm 0.046$ & $0.550 \pm 0.054$ \\
    $\mathcal{R}_{\text{log-det}}(V)$   & $0.722 \pm 0.114$ & $0.462 \pm 0.021$ & $\underline{0.628 \pm 0.024}$ & $0.620 \pm 0.046$ & $0.534 \pm 0.050$ & $0.549 \pm 0.027$ & $0.563 \pm 0.066$ & $0.573 \pm 0.055$ \\
    $\mathcal{R}(V)$ by MU              & \underline{$0.800 \pm 0.077$} & $0.462 \pm 0.018$ & $0.626 \pm 0.023$ & $0.615 \pm 0.092$ & \underline{$0.541 \pm 0.047$} & \underline{$0.562 \pm 0.035$} & \underline{$0.592 \pm 0.049$} & \underline{$0.587 \pm 0.058$} \\
    Proposed $\mathcal{R}(V)$           & \pmb{$0.835 \pm 0.068$} & \pmb{$0.468 \pm 0.017$} & \pmb{$0.630 \pm 0.018$} & \underline{$0.634 \pm 0.089$} & \pmb{$0.568 \pm 0.030$} & \pmb{$0.717 \pm 0.052$} & \pmb{$0.640 \pm 0.047$} & \pmb{$0.623 \pm 0.053$} \\
    \bottomrule
    \toprule
    NMI & SEEDS & YaleB & ORL & CHART & USPS-1000 & COIL20 & MNIST-2000 & Semeion \\
    \midrule
    SymNMF                              & $0.400 \pm 0.147$ & $0.522 \pm 0.017$ & $0.789 \pm 0.010$ & $0.710 \pm 0.066$ & $0.529 \pm 0.030$ & $0.689 \pm 0.021$ & $0.523 \pm 0.026$ & $0.553 \pm 0.043$ \\
    PHALS                               & $0.551 \pm 0.143$ & $0.526 \pm 0.012$ & $0.787 \pm 0.007$ & \underline{$0.767 \pm 0.041$} & $0.553 \pm 0.028$ & \underline{$0.743 \pm 0.027$} & $0.562 \pm 0.042$ & $0.585 \pm 0.036$ \\
    $\mathcal{R}_{\text{off-diag}}(V)$  & $0.429 \pm 0.114$ & \pmb{$0.535 \pm 0.011$} & $0.787 \pm 0.006$ & $0.714 \pm 0.048$ & $0.549 \pm 0.038$ & $0.695 \pm 0.024$ & $0.541 \pm 0.040$ & $0.541 \pm 0.035$ \\
    $\mathcal{R}_{\text{log-det}}(V)$   & $0.498 \pm 0.124$ & $0.529 \pm 0.012$ & \underline{$0.793 \pm 0.010$} & $0.725 \pm 0.058$ & $0.552 \pm 0.036$ & $0.725 \pm 0.018$ & $0.556 \pm 0.040$ & $0.580 \pm 0.043$ \\
    $\mathcal{R}(V)$ by MU              & \underline{$0.557 \pm 0.094$} & $0.527 \pm 0.013$ & $0.791 \pm 0.012$ & $0.711 \pm 0.064$ & \underline{$0.560 \pm 0.026$} & $0.739 \pm 0.018$ & \underline{$0.568 \pm 0.038$} & \underline{$0.588 \pm 0.027$} \\
    Proposed $\mathcal{R}(V)$           & \pmb{$0.634 \pm 0.068$} & \underline{$0.532 \pm 0.017$} & \pmb{$0.796 \pm 0.009$} & \pmb{$0.784 \pm 0.033$} & \pmb{$0.593 \pm 0.030$} & \pmb{$0.868 \pm 0.017$} & \pmb{$0.656 \pm 0.029$} & \pmb{$0.637 \pm 0.022$} \\
    \bottomrule
    \end{tabular}
\end{table*}

\begin{figure*}[!t]
    \centering
    \includegraphics[width=\linewidth]{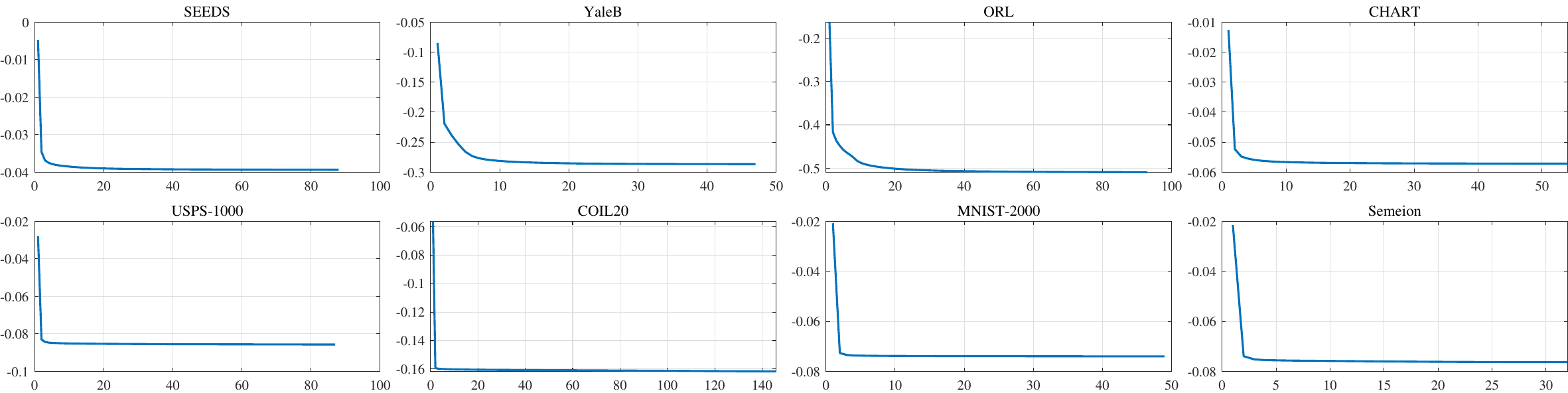}
    \caption{Convergence curves of the proposed model \eqref{eq:model} on eight datasets. For each dataset, the $x$-axis represents the iteration count, and the $y$-axis represents the objective function of \eqref{eq:model}.}
    \label{fig:convergence}
\end{figure*}

The clustering performances of these models are shown in Table \ref{tab:ablation_AOSymNMF}.
It can be seen that orthogonality regularization can significantly improve the clustering performance.
Among the MU-based methods, the proposed $\mathcal{R}(V)$ has a slight advantage compared with $\mathcal{R}_{\text{off-diag}}(V)$ and $\mathcal{R}_{\text{log-det}}(V)$.
Benefiting from the effectiveness of the PAHLS algorithm, the performance of the proposed $\mathcal{R}(V)$ is further improved.
As a result, the proposed $\mathcal{R}(V)$ demonstrates significant advantages among all methods.

\subsection{Convergence Analysis}

The theoretical convergence guarantee is a key advantage of our proposed orthogonality regularization, as analyzed in Theorem \ref{theorem:convergence_loss} and Theorem \ref{theorem:KKT}.
In this subsection, we empirically verify the empirical convergence behavior in Fig. \ref{fig:convergence}, where the hyperparamaters $\mu = \alpha = 0.1, \beta = 10$.
It can be found that the objective function is monotonically decreases, and typically converges within $100$ iterations.

\section{Conclusion} \label{sec:6}
In this paper, we proposed a novel SymNMF model that constructs the similarity graph and dissimilarity graph as weighted $k$-NN graph, and adaptively learns the weights and clustering results simultaneously.
Moreover, a new orthogonality regularization with explicit geometric meaning and good numerical stability is proposed.
The proposed model is solved by an alternative optimization algorithm with theoretical convergence guarantee.
Extensive experimental results demonstrate that the proposed model can achieve excellent clustering performance.

\bibliographystyle{ieeetr}
\bibliography{references}

\end{document}